\documentclass[nohyperref]{article}

\usepackage{microtype}
\usepackage{graphicx}
\usepackage{booktabs} 
\usepackage{subfig}
\usepackage{float}

\usepackage{hyperref}

\usepackage[accepted]{icml2022}

\usepackage{amsmath}
\usepackage{amssymb}
\usepackage{mathtools}
\usepackage{amsthm}

\usepackage{mymacros}

\usepackage{enumitem}
\usepackage{hyperref}

\usepackage{graphicx}
\usepackage{pgfplots}
\pgfplotsset{compat=newest}
\usepgfplotslibrary{groupplots}

\usepackage{comment}

\allowdisplaybreaks

\usepackage[capitalize,noabbrev]{cleveref}

\theoremstyle{plain}

\newcommand{\algrested}{\texttt{R-ed-UCB}\@\xspace}
\newcommand{\algrestless}{\texttt{R-less-UCB}\@\xspace}
\newcommand{\rless}{\texttt{R-less}\@\xspace}
\newcommand{\red}{\texttt{R-ed}\@\xspace}

\allowdisplaybreaks[4]

\definecolor{color12}{RGB}{0, 153, 136}
\definecolor{color11}{RGB}{238, 119, 51}
\definecolor{black}{rgb}{0, 0, 0}

\newcommand{\hl}[1]{{\color{black}#1}}

\icmltitlerunning{Stochastic Rising Bandits}

\begin{document}

\twocolumn[
\icmltitle{Stochastic Rising Bandits}




\begin{icmlauthorlist}
\icmlauthor{Alberto Maria Metelli}{polimi}
\icmlauthor{Francesco Trovò}{polimi}
\icmlauthor{Matteo Pirola}{polimi}
\icmlauthor{Marcello Restelli}{polimi}
\end{icmlauthorlist}

\icmlaffiliation{polimi}{Dipartimento di Elettronica, Informazione e Bioingegneria, Politecnico di Milano, Milan, Italy}

\icmlcorrespondingauthor{Alberto Maria Metelli}{\href{mailto:albertomaria.metelli@polimi.it}{\texttt{albertomaria.metelli@polimi.it}}}

\icmlkeywords{Machine Learning, ICML}

\vskip 0.3in
]



\printAffiliationsAndNotice{} 

\begin{abstract}
This paper is in the field of stochastic Multi-Armed Bandits (MABs), \ie those sequential selection techniques able to learn online using only the feedback given by the chosen option (a.k.a.~arm). We study a particular case of the rested and restless bandits in which the arms' expected payoff is monotonically non-decreasing. This characteristic allows designing specifically crafted algorithms that exploit the regularity of the payoffs to provide tight regret bounds. We design an algorithm for the rested case (\algrested) and one for the restless case (\algrestless), providing a regret bound depending on the properties of the instance and, under certain circumstances, of $\widetilde{\mathcal{O}}(T^{\frac{2}{3}})$. We empirically compare our algorithms with state-of-the-art methods for non-stationary MABs over several synthetically generated tasks and an online model selection problem for a real-world dataset.
Finally, using synthetic and real-world data, we illustrate the effectiveness of the proposed approaches compared with state-of-the-art algorithms for the non-stationary bandits.
\end{abstract}

\section{Introduction}
The classical stochastic MAB framework~\citep{lattimore2020bandit} has been successfully applied to a number of applications, such as advertising, recommendation, and networking.
MABs model the scenario in which a learner sequentially selects (a.k.a.~pulls) an option (a.k.a.~arm) in a finite set, and receives a feedback (a.k.a.~reward) corresponding to the chosen option.
The goal of online learning algorithms is to guarantee the \emph{no-regret} property, meaning that the loss due to not knowing the best arm is increasing sublinearly with the number of pulls.
One of the assumptions that allows designing no-regret algorithms consists in requiring that the payoff (a.k.a.~expected reward) provided by the available options is \emph{stationary}, \ie rewards come from a fixed distribution.

However, the arms' payoff may change over time due to intrinsic modifications of the arms or the environment.
A no-regret approach is offered by the \emph{adversarial} algorithms, in which no assumption on the nature of the reward is required.
It has been shown that, in this setting, it is possible to design effective algorithms, \eg \texttt{EXP3}~\citep{auer1995gambling}. However, in practice, their performance is unsatisfactory because the \emph{non-stationarity} of real-world cases is far from being adversarial. Instead, non-stationarity is explicitly accounted for by a surge of methods that consider either abrupt changes~\citep[\eg][]{garivier2011upper}, smoothly changing environments~\citep[\eg][]{trovo2020sliding} or bounded reward variation~\citep[\eg][]{besbes2014stochastic}.

While in non-stationary MABs the arms' payoff changes \emph{naturally} over time, a different setting arises when the payoff changes as an effect of \emph{pulling} the arm. This is the case of \emph{rotting} bandits~\citep{levine2017rotting, seznec2019rotting}, in which the payoff of the arms are monotonically non-increasing over the pulls, modeling degradation phenomena. Knowing the monotonicity property allows deriving more specialized algorithms, exploiting the process characteristics and further decreasing the regret w.r.t.~unrestricted cases. Notably, the symmetric problem of monotonically non-decreasing payoffs cannot be addressed with the same approaches. Indeed, it was shown that it represents a significantly more complex problem, even for deterministic arms~\citep{heidari2016tight}. In this non-decreasing setting, a common assumption is the concavity of the payoff function that defines the \emph{rising} bandits setting~\citep{li2020efficient}.

The goal of this paper is to study the \emph{stochastic} MAB problem when the arms' payoff is monotonically non-decreasing. This setting arises in several real-world sequential selection problems. For instance, suppose we have to choose among a set of optimization algorithms to maximize an unknown stochastic concave function. In this setting, we expect that all the algorithms progressively \emph{increase} (on average) the function value and eventually converge to an optimal value, possibly with different speeds. Therefore, we wonder which candidate algorithm to assign the available resources (\eg computational power or samples) to identify the one that converges faster to the optimum. This \emph{online model selection} process can be modeled as a \emph{rested} MAB~\citep{tekin2012online}, like the rotting bandits~\citep{levine2017rotting}, but with non-decreasing payoffs. Indeed, each optimization algorithm (arms) and the function value does not evolve if we do not select (pull) it.
Another example that shows a non-decreasing expected reward is the selection of athletes for competitions. Athletes train in parallel and increase (on average) their performance. However, if participation in competitions is allowed to one athlete only, the trainer should select the one who has achieved the best performance so far. This problem is akin to the \emph{restless} case~\citep{tekin2012online}, like non-stationary bandits~\citep{besbes2014stochastic}, but with the additional assumption that payoffs are non-decreasing. Indeed, the athletes (arms) are evolving even if they are not selected (pulled) to compete.

\textbf{Original Contribution}~~In this paper, we study the stochastic rising bandits, \ie stochastic bandits in which the payoffs are monotonically non-decreasing and concave, in both restless and rested formulations. More specifically:
\begin{itemize}[noitemsep, topsep=0pt, leftmargin=*]
	\item we show that the rested bandit with non-decreasing payoffs is \emph{non-learnable}, \ie the loss due to learning is linear with the number of pulls, unless additional assumptions on the payoff functions are enforced (\eg concavity);
	\item we design \algrested and \algrestless, optimistic algorithms for the rising rested and restless bandits;
	\item we show that \algrested and \algrestless suffer an expected regret that depends on the payoff function profile and, under some conditions, of order $\widetilde{\mathcal{O}}(T^{\frac{2}{3}})$;\footnote{With $\widetilde{\mathcal{O}}(\cdot)$ we disregard logarithmic terms in the order.}
	\item we illustrate, using synthetic and real-world
	data, the effectiveness of our approaches, compared with state-of-the-art algorithms for the non-stationary (restless) bandits.
\end{itemize}

\section{Related Works}
\textbf{Restless and Rested Bandits}~~The \emph{rested} and \emph{restless} bandit settings have been introduced by~\citet{tekin2012online} and further developed by~\citep{ortner2012regret,russac2019weighted} in the restless version and by~\citep{mintz2020nonstationary,pike2019recovering} in the rested one. Originally the evolution of the payoff was modeled via a suitable process, \eg a Markov chain with finite state space or a linear regression process. For instance,~\citet{NEURIPS2020_89ae0fe2} proposes an optimistic approach based on the estimation of the transition kernel of the underlying chain. More recently, the terms rested and restless have been employed to denote arms whose payoff changes as time passes, for restless ones, or whenever being pulled, for rested ones~\citep{seznec2019rotting,seznec2020asingle}. That is the setting we target in this work.

\textbf{Non-Stationary Bandits}~~The restless bandits, without a fixed temporal reward evolution, are usually addressed via non-stationary MAB approaches, that include both passive~\citep[\eg][]{garivier2011upper,besbes2014stochastic,auer2019adaptively,trovo2020sliding}
and active~\citep[\eg][]{liu2018change,besson2019efficient,cao2019nearly} methods. The former algorithms base their selection criterion on the most recent feedbacks, while the latter actively try to detect if a change in the arms' rewards occurred and use only data gathered after the last change. \citet{garivier2011upper} employ a discounted reward approach (\texttt{D-UCB}) or an adaptive sliding window (\texttt{SW-UCB}), proving a $\widetilde{\mathcal{O}}(\sqrt{T})$ regret when the number of abrupt changes is known. Similar results have been obtained by~\cite{auer2019adaptively} without knowing the number of changes, at the price of resorting to the doubling trick. \cite{besbes2014stochastic} provides an algorithm, namely \texttt{RExp3}, a modification \texttt{EXP3}, originally designed for adversarial MABs, to give a regret bound of $\mathcal{O}(T^{\frac{2}{3}})$ under the assumption that the total variation $V_T$ of the arms' expected reward is known. The knowledge of $V_T$ has been removed by~\citet{chen2019anew} using the doubling trick. In~\citet{trovo2020sliding}, an approach in which the combined use of a sliding window on a Thompson Sampling-like algorithm provides theoretical guarantees both on abruptly and smoothly changing environments. Nonetheless, in our setting, their result might lead to linear regret for specific instances. Notably, none of the above explicitly use assumptions on the monotonicity of the payoff over time.

\textbf{Rising Bandits}~~The \emph{rising} bandit problem has been tackled in its deterministic version by~\citep{heidari2016tight,li2020efficient}. In~\citet{heidari2016tight}, the authors design an online algorithm to minimize the regret of selecting an increasing and concave function among a finite set. This study assumes that the learner receives feedback about the true value of the reward function, \ie no stochasticity is present. In~\citet{li2020efficient}, the authors model the problem of parameter optimization for machine learning models as a rising bandit setting. They propose an online algorithm having good empirical performance, still in the case of deterministic rewards. A case where the reward is increasing in expectation (or equivalently decreasing in loss), but no longer deterministic, is provided by~\citet{cella2021best}.
However, the payoff follows a given parametric form known to the learner, who estimates such parameters in the best-arm identification and regret-minimization frameworks. The need for knowing the parametric form of the payoff makes these approaches hardly applicable for arbitrary increasing functions.

\textbf{Corralling Bandits}~~It is also worth mentioning the \emph{corralling} bandits~\citep{agarwal2017corralling, pacchiano2020model, abbasi2020regret, pacchiano2020regret, arora2021corralling}, a setting in which the goal is to minimize the regret of a process choosing among a finite set of bandit algorithms. This setting, close to online model selection, is characterized by particular assumptions. Indeed, each arm corresponds to a learning algorithm, operating on a bandit, endowed with a (possibly known) regret bound, sometimes requiring additional conditions (\eg stability).

\section{Problem Setting}\label{sec:problemSetting}
A $K$-armed MAB~\citep{lattimore2020bandit} is defined as a vector of probability distributions $\bm{\nu} = (\nu_i)_{i \in [K]}$, where $\nu_i : \Nat^2 \rightarrow \Delta(\Reals)$ depends on a pair of parameters $(t,n) \in \Nat^2$ for every $i \in [K]$, where $[K] \coloneq \{1, \ldots, K\}$. Let $T \in \Nat$ be the optimization horizon, at each round $t \in [T]$, the agent selects an arm $I_t \in [K]$ and observes a reward $R_{t} \sim \nu_{I_{t}}(t,N_{I_t,t})$, where $N_{i,t} = \sum_{l=1}^{t} \Ind\{I_l=i\}$ is the number of times arm $i\in [K]$ was pulled up to round $t$. Thus, the reward depends, in general, on the current round $t$ and on the number of pulls $N_{I_t,t} = N_{I_t,t-1}+1$ of arm $I_t$ up to $t$.
For every arm $i \in [K]$, we define its payoff $\mu_i : \Nat^2 \rightarrow \Reals$ as the expectation of the reward, \ie $\mu_i(t,n) = \E_{R \sim \nu_{i}(t,n)}[R]$ and denote the vector of payoffs as $\bm{\mu} = (\mu_i)_{i \in [K]}$. We assume that the payoffs are bounded in $[0,1]$, and that the rewards are $\sigma^2$-subgaussian, \ie $\E_{R \sim \nu_{i}(t,n)}[e^{\lambda  (R-\mu_i(t,n))}] \le e^{\frac{\sigma\lambda^2}{2}}$, for every $\lambda \in \Reals$.

%

\textbf{Rested and Restless Arms}~~We revise the definition of \emph{rested} and \emph{restless} arms~\citep{tekin2012online}.\footnote{We refer to the definition of~\citep{levine2017rotting,seznec2020asingle} and not to the one of~\citep{tekin2012online} that assumes an underlying Markov chain governing the arms' distributions.}
\begin{defi}[Rested and Restless Arms]
Let $\bm{\nu}$ be a MAB and let $i \in [K]$ be an arm, we say that:
\begin{itemize}[noitemsep, topsep=0pt, leftmargin=*]
\item $i$ is a \emph{rested} arm if, for every round $t \in [T]$ and number of pulls $n \in \Nat$, we have $\mu_i(t,n)  = \mu_i(n)$;
	\item $i$ is a \emph{restless} arm if, for every round $t \in [T]$ and number of pulls $n \in \Nat$, we have $\mu_i(t,n) = \mu_i(t)$.
\end{itemize}
A $K$-armed bandit is rested (resp.~restless) if all of its arms are rested (resp.~restless).
\end{defi}
Thus, the payoff of a rested arm changes when being pulled and, therefore, it models phenomena that evolve as a consequence of the agent intervention. Instead, a restless arm is in all regards a non-stationary arm~\citep{besbes2014stochastic}, and it is suitable for modeling a natural phenomenon that evolves for time passing, independently of the agent intervention.

\textbf{Rising Bandits}~~We revise the \emph{rising} bandits notion, \ie MABs with  payoffs \emph{non-decreasing} and \emph{concave} as a function of $(t,n)$~\citep{heidari2016tight}.\footnote{Deterministic bandits with non-decreasing payoffs were introduced in~\citep{heidari2016tight} with the term \emph{improving}. In~\citep{li2020efficient}, the term \emph{rising} was used to denote the improving bandits with concave payoffs (concavity was already employed by~\citet{heidari2016tight}).}
\begin{ass}[Non-Decreasing Payoff]\label{ass:incr}
Let $\bm{\nu}$ be a MAB, for every arm $i\in [K]$, number of pulls $n \in \Nat$, and round $t \in [T]$, functions $\mu_i(\cdot,n)$ and $\mu_i(t,\cdot)$ are non-decreasing. In particular, we define the \emph{increments}:
\begin{center}
	\begin{tabular}{ll}
	Rested arm: & $\gamma_i(n) \coloneqq \mu_i(n+1) - \mu_i(n) \ge 0$;\\
		Restless arm: & $\gamma_i(t) \coloneqq \mu_i(t+1) - \mu_i(t) \ge 0$.
	\end{tabular}
\end{center}
\end{ass}
From an economic perspective, $\gamma_i(\cdot)$ represents the \emph{increase of total return} (or payoff) we obtain by adding a factor of production, \ie pulling the arm (rested) or letting time evolve for a unit (restless).
In the next sections, we analyze how the following assumption defines a remarkable class of bandits with non-decreasing payoffs~\citep{heidari2016tight}.

\begin{ass}[Concave Payoff]\label{ass:decrDeriv}
Let $\bm{\nu}$ be a MAB, for every arm $i\in [K]$, number of pulls $n \in \Nat$, and round $t \in [T]$, functions $\mu_i(\cdot,n)$ and $\mu_i(t,\cdot)$ are concave, i.e.:
\begin{center}
\begin{tabular}{ll}
	Rested arm: & $\gamma_i(n+1) - \gamma_i(n) \le 0$;\\
	Restless arm: & $\gamma_i(t+1) - \gamma_i(t) \le 0$.
\end{tabular}
\end{center}
\end{ass}

As pointed out by~\citet{heidari2016tight}, the concavity assumption corresponds, in economics, to the \emph{decrease of marginal returns} that emerges when adding a factor of production, \ie pulling the arm (rested) or letting time evolve for one unit (restless).

Formally, we define \emph{rising} a stochastic MAB in which both Assumption~\ref{ass:incr} and Assumption~\ref{ass:decrDeriv} hold.

\textbf{Learning Problem}~~Let $t \in [T]$ be a round, we denote with $\Hs_t = (I_l,R_l)_{l=1}^t$ the \emph{history} of observations up to $t$. A (non-stationary) deterministic policy is a function $\pi : \Hs_{t-1} \mapsto I_t$ mapping a history to an arm, that is abbreviated as $\pi(t) \coloneqq \pi(\Hs_{t-1})$. The performance of a policy $\pi$ in a MAB with payoffs $\bm{\mu}$ is the \emph{expected cumulative reward} collected over the $T$ rounds, formally:
\begin{align*}
J_{\bm{\mu}}({\pi},T) \coloneq  \E \bigg[\sum_{t \in [T]} \mu_{I_t} \left(t, N_{I_t,t}\right) \bigg],
\end{align*}
and the expectation is computed over the histories.
A policy ${\pi}^*_{\bm{\mu},T}$ is \emph{optimal} if it maximizes the expected cumulative reward:	${\pi}^*_{\bm{\mu},T} \in \argmax_{\pi} \{ J_{\bm{\mu}}({\pi},T) \}$. Denoting with $J_{\bm{\mu}}^*\left(T\right) \coloneq J_{\bm{\mu}}({\pi}^*_{\bm{\mu},T},T)$ the expected cumulative reward of an optimal policy, the suboptimal policies $\pi$ are evaluated via the \emph{expected cumulative regret}:
\begin{align}\label{eq:regret}
	R_{\bm{\mu}}(\pi,T) \coloneq J_{\bm{\mu}}^*\left(T\right) - J_{\bm{\mu}}\left(\pi,T\right).
\end{align}

\textbf{Problem Characterization}~~
To characterize the problem instance, we introduce the following quantity, namely the \emph{cumulative increment}, defined for every $M \in [T]$ and $q \in [0,1]$ as:\footnote{The definition of cumulative increment was incorrect in the conference version of the paper \citep{MetelliTPR22}.}
\begin{align}\label{eq:magicQuantity}
	& \Upsilon_{\bm{\mu}}(M,q) \coloneq  \sum_{l=1}^{M-1} \max_{i \in [K]}\{\gamma_i(l)^q\}.
\end{align}
The cumulative increment accounts for how fast the payoffs reach their asymptotic value, \ie become stationary. Intuitively, small values of $\Upsilon_{\bm{\mu}}(M,q)$ lead to simpler problems, as they are closer to stationary bandits. Table~\ref{tab:rates} reports some bounds on $\Upsilon_{\bm{\mu}}(M,q)$ for particular choices of $\gamma_i(l)$ and $q$. When $q=1$, the cumulative increment $\Upsilon_{\bm{\mu}}(T,1)$ corresponds to the total variation $V_T \coloneq \sum_{l=1}^{T-1} \max_{i \in [K]} \left\{\gamma_i(l) \right\}$~\citep{besbes2014stochastic}.


\begin{table}
\renewcommand{\arraystretch}{1.5}
\caption{$\mathcal{O}$ rates of $\Upsilon_{\bm{\mu}}(M,q)$ in the case $\gamma_i(l) \le f(l)$ for all $i \in [K]$ and $l \in \Nat$ (see also Lemma~\ref{lemma:lemmaBoundsGamma}).}\label{tab:rates}
\vskip 0.15in 
\small
\begin{tabular}{r|C{1cm}|C{1.2cm}|C{1.2cm}|C{1.2cm}|}
\cline{2-5}\renewcommand{\arraystretch}{1}
$f(l)$ &  $e^{-cl}$ & $l^{-c}$ ($cq > 1$) & $l^{-c}$ ($cq = 1$) & $l^{-c}$ ($cq \le 1$)\\\cline{2-5}
  \rule{0pt}{15pt} $\Upsilon_{\bm{\mu}}(M,q)$ & $\displaystyle \frac{e^{-cq}}{cq}$ & $\displaystyle \frac{1}{cq-1}$ &  $\displaystyle \log M$ & $\displaystyle \frac{M^{1-cq}}{1-cq}$ \\\cline{2-5}
\end{tabular}
\end{table}

In the next sections, we devise and analyze learning algorithms for rested (Section~\ref{sec:rresed}) and restless (Section~\ref{sec:rrestless}) rising bandits. We will present \emph{optimistic} algorithms, whose structure is summarized in Algorithm~\ref{alg:alg} and parametrized by an exploration index $B_i(t)$ that will be designed case by case.
%
%
%

\section{Stochastic Rising Rested Bandits}\label{sec:rresed}
In this section, we consider the \emph{\texttt{R}ising rest\texttt{ed}} bandits (\red) setting in which the arms' expected payoff increases only when it is pulled, \ie $\mu_i(t,N_{i,t}) \equiv \mu_i(N_{i,t})$.\footnote{We are employing the original definition of rested arms of~\cite{levine2017rotting} in which $\mu_i(n)$ is the payoff of arm $i$ when it is pulled \emph{for the $n$-th time}. }

\textbf{Oracle Policy}~~We recall that the \emph{oracle constant} policy, that always plays at each round $t \in [T]$ the arm that maximizes the sum of the payoffs over the horizon $T$, is optimal for the non-decreasing rested bandits.

\begin{restatable}[\citealp{heidari2016tight}]{thr}{thrRestedOptimal}\label{thr:thrRestedOptimal}
	Let $\pi^c_{\bm{\mu},T}$ be the \emph{oracle constant} policy:
	\begin{align*}
		\pi^c_{\bm{\mu},T}(t) \in \argmax_{i \in [K]} \Bigg\{ \sum_{l \in [T]} \mu_i(l)\Bigg\}, \quad \forall t \in [T].
	\end{align*}
	Then, $\pi^c_{\bm{\mu},T}$ is optimal for the rested non-decreasing bandits (\ie under Assumption~\ref{ass:incr}).
\end{restatable}

The result holds under the non-decreasing property (Assumption~\ref{ass:incr}) only, without requiring concavity (Assumption~\ref{ass:decrDeriv}). However, this policy cannot be used in practice as it requires knowing the full function $\mu_i(\cdot)$ in advance.

\subsection{Non-Learnability}\label{sec:restedNon} 
We now prove a result highlighting the \quotes{hardness} of the non-decreasing rested bandits. We show that with no assumptions on the payoff $\mu_i(n)$ (\eg concavity), it is impossible to devise a no-regret algorithm.
\begin{restatable}[Non-Learnability]{thr}{nonLearnable}\label{thr:nonLearnable}
	There exists a $2$-armed non-decreasing (non-concave) deterministic rested bandit with $\gamma_i(n) \le \gamma_{\max} \le 1$ for all $i \in [K]$ and $n \in \Nat$, such that any learning policy $\pi$ suffers regret:
	\begin{align*}
		R_{\bm{\mu}}(\pi, T) \ge \left\lfloor \frac{\gamma_{\max}}{12} T \right\rfloor. 
	\end{align*}
\end{restatable}
\hl{The intuition behind this result is that, if we enforce no condition on the increment $\gamma_i(n)$ we cannot predict how much the arm payoff will increase in the future. Therefore, we face the dilemma of whether or not to pull an arm that is currently believed to be suboptimal, hoping its payoff will increase. If we decide to pull it and its payoff will not actually increase, or if we decide not to pull it and its payoff will actually increase, becoming optimal, we will suffer linear regret.} Thus, Theorem~\ref{thr:nonLearnable} highlights the importance of the concavity assumption (Assumption~\ref{ass:decrDeriv}), providing an answer to an open question posed in~\citep{heidari2016tight}.

\begin{algorithm}[t]
	\begin{algorithmic}
	\STATE \textbf{Input}: $K$, $(B_i)_{i \in [K]}$
	\STATE Initialize $N_{i} \leftarrow 0$ for all $i \in [K]$
	\FOR{$t \in ( 1, \dots, T)$}{
		\STATE Pull $I_t \in \argmax_{i \in [K]} \{B_i(t)\}$
		\STATE Observe $R_t \sim \nu_{I_t}(t, N_{I_t}+1)$
		\STATE Update $B_{I_t}$ and $N_{I_t} \leftarrow N_{I_t} + 1$ 
	}
	\ENDFOR
	\end{algorithmic}
	\caption{\texttt{R-$ \boxvoid$-UCB}  ($ \boxvoid \in \{\text{\texttt{less},\texttt{ed}}\}$ )}\label{alg:alg}
\end{algorithm}

\begin{remark}[About the Concavity Assumption]
While without additional structure, \eg concavity, the \emph{non-decreasing rested} bandits are non-learnable (Theorem~\ref{thr:nonLearnable}), the assumption is not necessary in other related settings. In particular, \emph{non-decreasing restless} bandits are in all regard non-stationary bandits, for which no-regret algorithms exist under different assumptions about the number of change points~\citep{garivier2011upper} or a bounded total variation~\citep{besbes2014stochastic}. Furthermore, for \emph{non-increasing rested} (rotting) bandits~\citep{levine2017rotting}, a bounded payoff decrement between consecutive pulls is sufficient to devise a no-regret algorithm.
%
\end{remark}
%
%
%
%

\subsection{Deterministic Setting}\label{sec:restedDet}
To progressively introduce the core ideas, we begin with the case of deterministic arms ($\sigma = 0$). 
We devise an optimistic estimator of $\mu_i(t)$, namely $\overline{\mu}_i^{\text{\red}} (t)$, having observed the exact payoffs $(\mu_i(n))_{n=1}^{N_{i,t-1}}$. Differently from the rotting setting, these payoffs are an underestimation of $\mu_i(t)$. Therefore, we exploit the non-decreasing assumption (Assumption~\ref{ass:incr}) to derive the identity:
%
%
\begin{equation}\label{eq:estRestedDet}
\begin{aligned}
	\mu_i(t) =  \textcolor{color11}{\underbrace{\mu_i(N_{i,t-1})}_{\text{(most recent payoff)}}} + \textcolor{color12}{\underbrace{\sum_{n=N_{i,t-1}}^{t-1} \gamma_i(n)\textcolor{black}{.}}_{\text{(sum of future increments)}}}
\end{aligned}
\end{equation}
By exploiting the concavity (Assumption~\ref{ass:decrDeriv}), we upper bound the sum of future increments with the last experienced increment $\gamma_i(N_{i,t-1}-1)$ that is projected for the future $t-N_{i,t-1}$ pulls, leading to the following estimator:
\begin{align}\label{eq:estRestedDetEst}
\overline{\mu}_i^{\text{\red}}(t) & \coloneqq \hspace{-0.22cm} \textcolor{color11}{\underbrace{\mu_i(N_{i,t-1})}_{\text{(most recent payoff)}}} \hspace{-0.22cm} + \textcolor{color12}{(t - N_{i,t-1}) \hspace{-0.1cm} \underbrace{ \gamma_i(N_{i,t-1}-1)\textcolor{black}{,}}_{\text{(most recent increment)}}}
\end{align}
if $N_{i,t-1} \ge 2$ else $\overline{\mu}_i^{\text{\red}}(t) \coloneqq +\infty$. Figure~\ref{fig:estConstr} illustrates the construction of the estimator. The optimism of $\overline{\mu}_i^{\text{\red}}$ and a bias bound are proved in Lemma~\ref{lemma:lemmaDetRested}.

\begin{figure}
\centering
	\includegraphics[width=.9\linewidth]{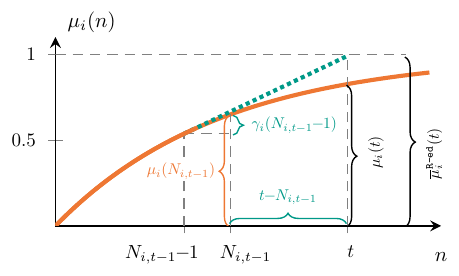}
	\caption{Graphical representation of the estimator construction $\overline{\mu}_i^{\text{\red}}(t)$ for the rested deterministic setting.}\label{fig:estConstr}
\end{figure}

%


\textbf{Regret Analysis}~~We are now ready to provide the regret analysis of \algrested, \ie Algorithm~\ref{alg:alg} when we employ as exploration index $B_i(t) \equiv \overline{\mu}_i^{\text{\red}}(t)$. 

\begin{restatable}[]{thr}{theRegretDetRested}\label{thr:theRegretDetRested}
Let $T \in \Nat$, then \algrested (Algorithm~\ref{alg:alg}) with $B_i(t) \equiv \overline{\mu}_i^{\text{\red}}(t)$ suffers an expected regret bounded, for every $q \in [0,1]$, as:
\begin{align*}
	R_{\bm{\mu}}(\text{\algrested},T) \le 2K + K T^q  \Upsilon_{\bm{\mu}}\left(\left\lceil \frac{T}{K} \right\rceil,q \right).
\end{align*}
\end{restatable}

The regret depends on a parameter $q \in [0,1]$ that can be selected to tighten the bound, whose optimal value depends on $\Upsilon_{\bm{\mu}}(\cdot,q)$, that is a function on the horizon $T$. Some examples, when $\gamma_i(t) \le l^{-c}$ for $c>0$, are reported in Figure~\ref{fig:rate}.

%
%
%

\subsection{Stochastic Setting}\label{sec:restedStoc}
Moving to the \red stochastic setting ($\sigma>0$), we cannot directly exploit the estimator in Equation~\eqref{eq:estRestedDetEst}. Indeed, we only observe the sequence of noisy rewards $(R_{t_{i,n}})_{n=1}^{N_{i,t-1}}$, where $t_{i,n} \in [T]$ is the round at which arm $i\in [K]$ was pulled for the $n$-th time. To cope with stochasticity, we need to employ an $h$-wide window made of the $h$ most recent samples, similarly to what has been proposed by~\citet{seznec2020asingle}. The choice of $h$ represents a \emph{bias-variance trade-off} between employing few recent observations (less biased), compared to many past observations (less variance). 
For $h \in [N_{i,t-1}]$, the resulting estimator $\widehat{\mu}_i^{\text{\red},h}(t)$ is given by:
\begin{align*}
	\widehat{\mu}_i^{\text{\red},h} (t)   \coloneqq \frac{1}{h}  \sum_{l=N_{i,t-1}-h+1}^{N_{i,t-1}} & \Bigg(\textcolor{color11}{\underbrace{R_{t_{i,l}}}_{\text{(estimated payoff)}}} \\
	&  + \textcolor{color12}{(t-l) \underbrace{\frac{R_{t_{i,l}} - R_{t_{i,l-h}}}{h}}_{\text{(estimated increment)}}}\Bigg),
\end{align*}
if $h \le \lfloor N_{i,t-1}/2 \rfloor$, else $\widehat{\mu}_i^{\text{\red},h} (t) \coloneqq +\infty$. The construction of the estimator is shown in Appendix~\ref{apx:prrested} and relies on the idea of averaging several estimators of the form of Equation~\eqref{eq:estRestedDetEst} instanced using as starting points different number of pulls $N_{i,t-1}-l+1$ for $l \in [h]$ and replacing the true payoff with the corresponding reward sample. An efficient way to compute this estimator is reported in Appendix~\ref{apx:efficient}.

\textbf{Regret Analysis}~~By making use of the presented estimator, we build the following optimistic exploration index:
\begin{align*}
	& B_{i}(t) \equiv \widehat{\mu}_i^{\text{\red},h_{i,t}} (t) + \beta_i^{\text{\red},h_{i,t}}(t), \quad \text{where} \\
& \beta^{\text{\red},h_{i,t}}_i(t,\delta_t)\coloneqq \sigma (t-N_{i,t-1}+h_{i,t}-1) \sqrt{ \frac{ 10 \log \frac{1}{\delta_t} }{h_{i,t}^3} },
\end{align*}
and $h_{i,t}$ are arm-and-time-dependent window sizes and $\delta_t$ is a time-dependent confidence parameter. By choosing the window size depending linearly on the number of pulls, we are able to provide the following regret bound.

\begin{restatable}[]{thr}{theRegretRestedStochastic}\label{thr:theRegretRestedStochastic}
Let $T \in \Nat$, then \algrested (Algorithm~\ref{alg:alg}) with $B_i(t) \equiv  \widehat{\mu}_i^{\text{\red},h_{i,t}}(t) + \beta_i^{\text{\red},h_{i,t}}(t)$,  $h_{i,t} = \left\lfloor \epsilon N_{i,t-1}\right\rfloor$ for $\epsilon \in (0,1/2)$ and $\delta_t = t^{-\alpha}$ for $\alpha > 2$, suffers an expected regret bounded, for every $q \in \left[0,  1 \right]$, as:
\begin{align*}
	& R_{\bm{\mu}}(\text{\algrested},T) \le \BigO \Bigg( \frac{K}{\epsilon} (\sigma  T)^{\frac{2}{3}} \left( \alpha  \log T \right)^{\frac{1}{3}} \\
	& \qquad\qquad+   \frac{K {T}^q }{1-2\epsilon}  \Upsilon_{\bm{\mu}}\left(\left\lceil  (1-2\epsilon) \frac{T}{K} \right\rceil,q \right)  \Bigg).
\end{align*}
\end{restatable}

\begin{figure}
\begin{minipage}{.5\linewidth}
\footnotesize
\renewcommand{\arraystretch}{2}
\setlength{\tabcolsep}{5pt}
\begin{tabular}{l|c|c|}
\multicolumn{1}{c}{} & \multicolumn{1}{c}{$c\le 1$} & \multicolumn{1}{c}{$c\ge 1$} \tabularnewline
\cline{2-3}
\textcolor{vibrantRed}{\textbf{Rested}} & $T$  & $K T^{\frac{1}{c}}$ \tabularnewline\cline{2-3}
\textcolor{vibrantBlue}{\textbf{Restless}} & $K^{\frac{1+c}{2}}T^{1-\frac{c}{2}}$ & $KT^{\frac{1}{c+1}}$ \tabularnewline\cline{2-3}
\end{tabular}
\vspace{1cm}
\end{minipage}
\hfill
\begin{minipage}{.4\linewidth}
\begin{tikzpicture}[
  declare function={
    frested(\x)= (\x < 1) * (1) + (\x >= 1) * (1 / (\x));
    frestless(\x)= (\x < 1) * (1- (\x)/2) + (\x >= 1) * (1 / (\x + 1));
  }
]
\begin{groupplot}[
  group style={
    group size=1 by 1,
    vertical sep=0pt,
    horizontal sep=0cm,
    group name=G},
     width=4cm,height=3.5cm,
  axis lines=middle,
  xlabel near ticks,
  every axis x label/.style={at={(current axis.right of origin)},anchor=west, below=1.5mm},
  legend style={at={(1.2,1)},anchor=north,legend cell align=left} 
  ]
  \nextgroupplot[xmin=0, xmax=5,ymax=1.1, domain=-0.1:5, samples=100, xtick = {0,1}, xticklabels = {$0$, $1$}, xlabel={$c$}]
  \addplot[vibrantBlue,  ultra thick] {frestless(x)};
  \addplot[vibrantRed,  ultra thick] {frested(x)};
  \addplot[mark=none, gray, dashed] coordinates {(1,0)(1,1)};
  
\end{groupplot}
\end{tikzpicture}
\end{minipage}
\caption{Regret bounds $\widetilde{\mathcal{O}}$ rates optimized over $q$ for \rless and \red deterministic bandits when $\gamma_i(l) \le l^{-c}$ for $c > 0$.}\label{fig:rate}
\end{figure}
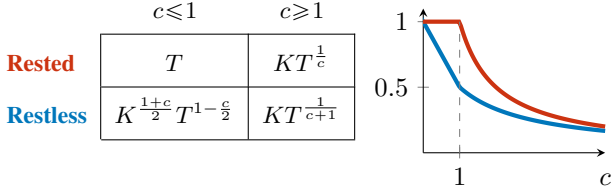

This result deserves some comments. First, compared with the corresponding deterministic \red regret bound (Theorem~\ref{thr:theRegretDetRested}), it reflects a similar dependence of the cumulative increment $\Upsilon_{\bm{\mu}}$, although it now involves the $\epsilon$ parameter defining the window size $h_{i,t} = \lfloor \epsilon N_{i,t-1} \rfloor$.  Second, it includes an additional term of order $\widetilde{\mathcal{O}}(T^{\frac{2}{3}})$ that is due to the noise $\sigma$ presence that increases inversely \wrt the $\epsilon$.\footnote{\hl{In particular, when $\gamma_i(n)$ decreases sufficiently fast (see Table~\ref{tab:rates}), the regret is dominated by the $\widetilde{\mathcal{O}}(T^{\frac{2}{3}})$ component.}} Thus, we visualize a trade-off in the choice of $\epsilon$: larger windows ($\epsilon \approx 1$) are beneficial for the first term, but they enlarge the constant $1/(1-2\epsilon)$ multiplying the second component.

\begin{remark}[Comparison with Adversarial Bandits]\label{remark:adversary}
\hl{The \red setting can be mapped to an \emph{adversarial} bandit~\cite{AuerCFS02} with an \emph{adaptive} (\ie non-oblivious) adversary. Indeed, the arm payoff $\mu_i(N_{i,t})$ can be thought to as selected by an adversary who has access to the previous learner choices (\ie the history $\mathcal{H}_{t-1}$), specifically to the number of pulls $N_{i,t}$. However, although adversarial bandit algorithms, such as  \texttt{EXP3}~\citep{AuerCFS02} and \texttt{OSMD}~\citep{AudibertBL14}, suffer $\widetilde{{\BigO}}({\sqrt{T}})$ regret, these results are not comparable with ours. Indeed, while these correspond to guarantees on the \emph{external regret}, the regret definition we employ in Section~\ref{sec:problemSetting} is a notion of \emph{policy regret}~\citep{DekelTA12}.

     }
\end{remark}

\section{Stochastic Rising Restless Bandits}\label{sec:rrestless}
In this section, we consider the \emph{\texttt{R}ising rest\texttt{less}} bandits (\rless) in which the payoff increases at every round regardless the arm is pulled, \ie $\mu_i(t,N_{i,t}) \equiv \mu_i(t)$.

\textbf{Oracle Policy}~~We start recalling that the \emph{oracle greedy} policy, \ie the policy selecting at each round $t \in [T]$ the arm with largest payoff, is optimal for the non-decreasing restless bandit setting.
\begin{restatable}[\citealp{seznec2020asingle}]{thr}{thrRestlessOptimal}
	Let $\pi^g_{\bm{\mu}}$ be the \emph{oracle greedy} policy:
	\begin{align*}
		\pi^g_{\bm{\mu}}(t) \in \argmax_{i \in [K]} \{\mu_i(t)\}, \quad \forall t \in [T].
	\end{align*}
	Then, $\pi^g_{\bm{\mu}}$ is optimal for the restless non-decreasing bandits (\ie under Assumption~\ref{ass:incr}).
\end{restatable}
Notice that $\pi^g_{\bm{\mu}}$ is optimal under the non-decreasing payoff assumption (Assumption~\ref{ass:incr}) only, without requiring the concavity (Assumption~\ref{ass:decrDeriv}). We can now first appreciate an important difference between \emph{rising} and \emph{rotting} bandits. While for the rotting bandits the oracle \emph{greedy} policy is optimal for both the rested and restless settings, for the rising bandits it remains optimal in the restless case only. Indeed, for the rising rested case, as shown in Theorem~\ref{thr:thrRestedOptimal}, the oracle \emph{constant} policy is needed to achieve optimality.
%

\subsection{Deterministic Setting}
We begin with the case of deterministic arms ($\sigma = 0$). 
Similarly to the rested case, we design an optimistic estimator of $\mu_i(t)$, namely $\overline{\mu}_i^{\text{\rless}} (t)$, employing the exact payoffs $(\mu_i(t_{i,n}))_{n=1}^{N_{i,t-1}}$. To this end, we exploit the non-decreasing assumption (Assumption~\ref{ass:incr}) to derive the identity:
\begin{align*}
	\mu_i(t) = \textcolor{color11}{\underbrace{\mu_i(t_{i,N_{i,t-1}})}_{\text{(most recent payoff)}}} + \textcolor{color12}{\underbrace{\sum_{l=t_{i,N_{i,t-1}}}^{t-1} \gamma_i(l)\textcolor{black}{.}}_{\text{(sum of future increments)}}}
\end{align*}
Then, we leverage the concavity (Assumption~\ref{ass:decrDeriv}) to upper bound the sum of future increments with the last experienced increment that will be projected in the future for $t - t_{i,N_{i,t-1}}$ rounds, leading to the estimator:
\begin{equation}\label{eq:estRestlessDet}
\begin{aligned}
& \overline{\mu}_i^{\text{\rless}} (t)  \coloneqq \textcolor{color11}{\underbrace{\mu_i(t_{i,N_{i,t-1}})}_{\text{(most recent payoff)}}} \\
		& \quad + \textcolor{color12}{(t - t_{i,N_{i,t-1}}) \underbrace{\frac{\mu_i(t_{i,N_{i,t-1}}) - \mu_i(t_{i,N_{i,t-1}-1})}{t_{i,N_{i,t-1}} - t_{i,N_{i,t-1}-1}}}_{\text{(most recent increment)}}},
\end{aligned} 
\end{equation}
if $N_{i,t-1} \ge 2$, else $\overline{\mu}_i^{\text{\rless}} (t) \coloneqq +\infty$.
Lemma~\ref{lemma:lemmaDetRestless} shows that $\overline{\mu}_i^{\text{\rless}}$ is optimistic and provides a bias bound.
%
%
%

\textbf{Regret Analysis} 
We now provide the regret analysis of \algrestless that is obtained from Algorithm~\ref{alg:alg}, when setting $B_i(t) \equiv \overline{\mu}_i^{\text{\rless}}(t)$. 
\begin{restatable}[]{thr}{theRegretDetRestless}\label{thr:theRegretDetRestless}
Let $T \in \Nat$, then \algrestless (Algorithm~\ref{alg:alg}) with $B_i(t) \equiv \overline{\mu}_i^{\text{\rless}}(t)$ suffers an expected regret bounded, for every $q \in \left[0,  1 \right]$, as:
\begin{align*}
	R_{\bm{\mu}}(\text{\algrestless},T) \le 2K +K T^\frac{q}{q+1} \Upsilon_{\bm{\mu}}\left( \left\lceil \frac{T}{K} \right\rceil, q \right)^{\frac{1}{q+1}}.
\end{align*}
\end{restatable}

Similarly to Theorem~\ref{thr:theRegretDetRested}, the result depends on the free parameter $q \in [0,1]$, that can be chosen to tighten the bound. It is worth noting that the regret bound of the \rless deterministic case (Theorem~\ref{thr:theRegretDetRestless}) is always smaller than that of the \red deterministic case (Theorem~\ref{thr:theRegretDetRested}). Indeed, ignoring the dependence on $K$, we have $R_{\bm{\mu}}(\text{\algrestless},T)  = \BigO\left(R_{\bm{\mu}}(\text{\algrested},T)^{\frac{1}{q+1}} \right)$. 
The following example clarifies the role of $q$ for both the restless and rested case. 

\begin{example}
Suppose that for all $i \in [K]$, we have $\gamma_i(l) \le l^{-c}$ for $c > 0$. The expressions of bounds on $\Upsilon_{\bm{\mu}}(\cdot,q)$ have been shown in Table~\ref{tab:rates}. Different values of $q \in [0,1]$ should be selected to tighten the regret bounds depending on the value of $c$. Figure~\ref{fig:rate} reports the optimized bounds for the deterministic \rless and \red (derivation in Appendix~\ref{apx:example}).
\end{example}

%

\subsection{Stochastic Setting}
In the stochastic setting ($\sigma > 0$), we have access to noisy versions of $\mu_i$ only, \ie $(R_{t_{i,n}})_{n=1}^{N_{i,t-1}}$. Intuitively, we might be tempted to straightforwardly extend the derivation of the rested case by averaging $h$ estimators like the ones in  Equation~\eqref{eq:estRestlessDet}, instanced with different time instants $t_{i,N_{i,t-1}}$. Unfortunately, this approach is unsuccessful for technical issues since the increment term would include the difference of time instants $t_{i,N_{i,t-1}} - t_{i,N_{i,t-1}-1}$ that, in the stochastic setting, are random variables correlated with the observed rewards $R_{t_{i, n}}$. For this reason, at the price of a larger bias, we employ the same estimator used in the stochastic rested case, defined for $h \in [N_{i,t-1}]$:
\begin{align*}
	\widehat{\mu}^{\text{\rless},h}_i(t)  & \coloneq  \frac{1}{h} \sum_{l=N_{i,t-1}-h+1}^{N_{i,t-1}} \bigg( \textcolor{color11}{\underbrace{R_{t_{i,l}}}_{\text{(estimated payoff)}}} \\
	& \quad\quad\quad + \textcolor{color12}{( t - l)\underbrace{\frac{R_{t_{i,l}} - R_{t_{i,l-h}}}{h}}_{\text{(estimated increment)}}} \bigg),
\end{align*}
if $h_{i,t} \le \lfloor N_{i,t-1}/2 \rfloor$, else $\widehat{\mu}^{\text{\rless},h}_i(t) \coloneq +\infty$. Additional details on the estimator construction is reported in Appendix~\ref{apx:prrestless} together with its analysis. 
\setlength{\textfloatsep}{10pt}
\begin{figure*}
	\subfloat{\scalebox{1}{\includegraphics{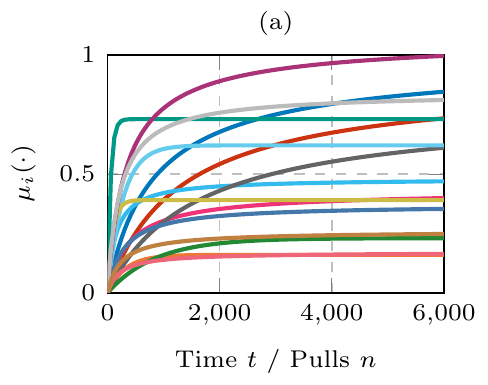} \label{fig:15arms_rewards}	}}
	\subfloat{\scalebox{1}{\includegraphics{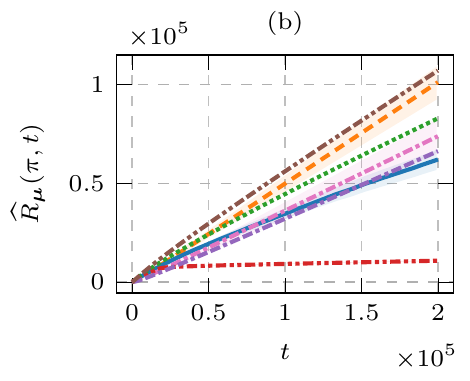} }\label{fig:restless_15arms_regrets}}
	\subfloat{\scalebox{1}{\includegraphics{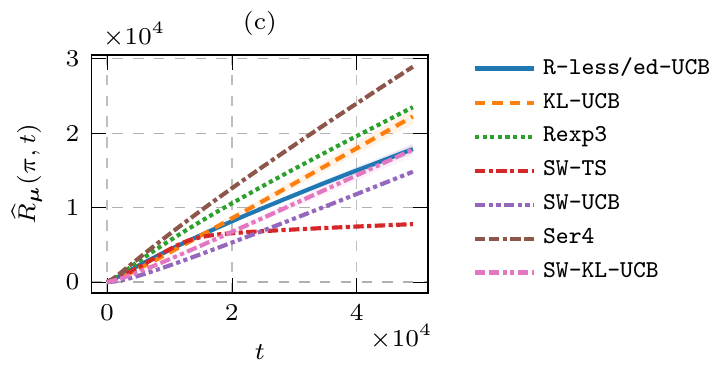}\label{fig:rested_15arms_regrets} }}
	\caption{$15$ arms bandit setting: (a)~first $6000$ rounds/pulls of the payoff functions, (b)~cumulative regret in the \rless scenario, (c)~cumulative regret in the \red scenario ($100$ runs $95$\% c.i.).}
\end{figure*}

\textbf{Regret Analysis}~~We provide the regret analysis of \algrestless when we employ the exploration index analogous to that of the rested case:
\begin{align*}
	& B_{i}(t) \equiv \widehat{\mu}_i^{\text{\rless},h_{i,t}} (t) + \beta_i^{\text{\rless},h_{i,t}}(t), \quad \text{where} \\
& \beta^{\text{\rless},h_{i,t}}_i(t,\delta_t)\coloneqq \sigma (t-N_{i,t-1}+h_{i,t}-1) \sqrt{ \frac{ 10 \log \frac{1}{\delta_t} }{h_{i,t}^3} },
\end{align*}
and $h_{i,t}$ are a arm-and-time-dependent window sizes and $\delta_t$ is a time-dependent confidence. The regret bound is given by the following result.
\begin{restatable}[]{thr}{theRegretRestlessStochastic}\label{thr:theRegretRestlessStochastic}
Let $T \in \Nat$, then \algrestless (Algorithm~\ref{alg:alg}) with $B_i(t) \equiv  \widehat{\mu}_i^{\text{\rless},h_{i,t}}(t) + \beta_i^{\text{\rless},h_{i,t}}(t)$,  $h_{i,t} = \left\lfloor \epsilon N_{i,t-1}\right\rfloor$ for $\epsilon \in (0,1/2)$, and $\delta_t = t^{-\alpha}$ for $\alpha > 2$, suffers an expected regret bounded, for every $q \in \left[ 0,  1 \right]$, as:
\begin{align*}
R_{\bm{\mu}}& (\text{\algrestless},T)  \le \BigO \Bigg(\frac{K}{\epsilon}(\sigma T)^{\frac{2}{3}} \left( \alpha \log T \right)^{\frac{1}{3}} \\
&  \qquad + \frac{KT^{\frac{2q}{1+q}} (\log T)^{\frac{q}{1+q}}}{\epsilon(1-2\epsilon)}    \Upsilon_{\bm{\mu}}\left( \left\lceil (1-2\epsilon) \frac{T}{K} \right\rceil,  q \right)^{\frac{1}{1+q}}  \Bigg).
\end{align*}
\end{restatable}
Some observations are in order. First, compared to the bound for the rested case in Theorem~\ref{thr:theRegretRestedStochastic}, we note the same dependence of $\widetilde{\mathcal{O}}(T^{\frac{2}{3}})$ due to the noise presence $\sigma$. Concerning the second term, compared with the one of the deterministic case (Theorem~\ref{thr:theRegretDetRestless}), we worsen the dependence on $T$ and an inverse dependence on the $\epsilon$ and $1-2\epsilon$ parameters appear. This is due to the usage of the $h$-wide window instead of the last sample and that, all other things being equal, the estimator employed for the stochastic case, as already discussed, is looser compared to the one for the deterministic case. Finally, our result is not fully comparable with~\citep{besbes2014stochastic} for generic non-stationary bandits with bounded variation because due to the presence of $\Upsilon_{\bm{\mu}}\left(\lceil T/K \rceil,q\right)$. Moreover, we achieve such a bound with no knowledge about $\Upsilon_{\bm{\mu}}$, while the work by~\cite{besbes2014stochastic} requires knowing $V_T$.
%
%

%
%
%
%


\section{Numerical Simulations} \label{sec:experiments}

\setlength{\textfloatsep}{10pt}
We numerically tested \algrestless{} and \algrested{} w.r.t.~state-of-the-art algorithms for non-stationary MABs in the \emph{restless} and \emph{rested} settings, respectively.\footnote{Details of the experimental setting, and additional results are provided in Appendix~\ref{apx:experiments}. The code to reproduce the experiments is available at \url{https://github.com/albertometelli/stochastic-rising-bandits}.}


\textbf{Algorithms}~~We consider the following baseline algorithms: \texttt{Rexp3}~\citep{besbes2014stochastic}, a non-stationary MAB algorithm based on  variation budget, \texttt{KL-UCB}~\citep{garivier2011kl}, one of the most effective stationary MAB algorithms, \texttt{Ser4}~\citep{allesiardo2017non}, which considers best arm switches during the process, and sliding-window algorithms such as \texttt{SW-UCB}~\citep{garivier2011upper}, \texttt{SW-KL-UCB}~\citep{combes2014unimodal}, and \texttt{SW-TS}~\citep{trovo2020sliding} that are generally able to deal with non-stationary restless settings.
The parameters for all the baseline algorithms have been set as recommended in the corresponding papers (see also Appendix~\ref{apx:experiments}). For our algorithms, the window is set as $h_{i,t} = \lfloor \epsilon N_{i,t-1} \rfloor$ (as prescribed by Theorems~\ref{thr:theRegretRestedStochastic} and~\ref{thr:theRegretRestlessStochastic}). We remark that while the baseline algorithms are suited for the restless case, in the rested case, no algorithm has been designed to cope with the stochastic rising setting, provided that no knowledge on the payoff function is available.
We compare the algorithms in terms of empirical cumulative regret $\widehat{R}_{\bm{\mu}}(\pi,t)$, which is the empirical counterpart of the expected cumulative regret $R_{\bm{\mu}}(\pi,t)$ at round $t$ averaged over multiple independent runs.

\begin{figure*}
\begin{minipage}{.57\textwidth}
\subfloat{\scalebox{1}{\includegraphics{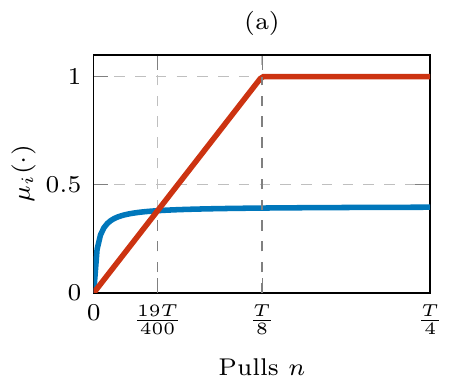}\label{fig:2arms_rewards} 	}}
	\subfloat{\scalebox{1}{\includegraphics{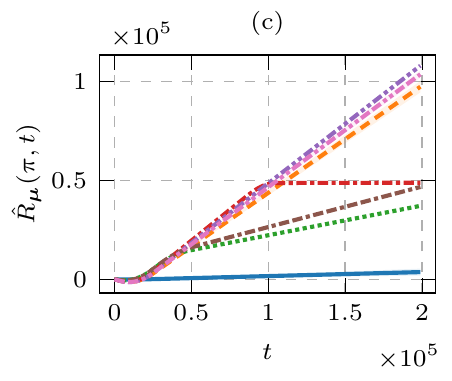} }\label{fig:rested_2arms_regrets}}
	\captionof{figure}{$2$ arms \red bandit setting: (a)~payoff functions, (b)~cumulative \\regret ($100$ runs, $95$\% c.i.).}
\end{minipage}%
\hfill
\begin{minipage}{.41\textwidth}
\vspace{.25cm}
\includegraphics{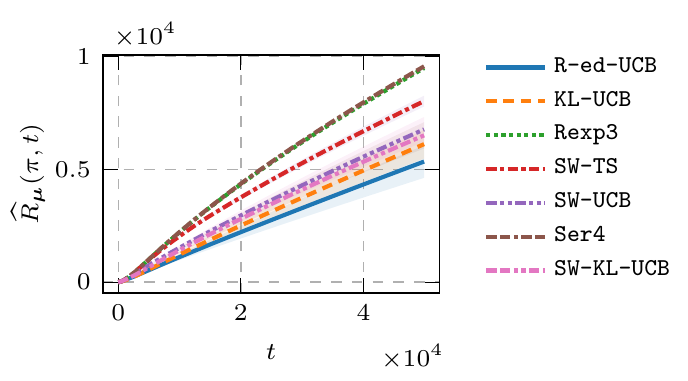} 
\captionof{figure}{Cumulative regret in the online model selection on IMDB dataset ($30$ runs, $95$\% c.i.).}\label{fig:imdb}
\end{minipage}
\end{figure*}
%

\subsection{Restless setting} \label{sec:restless}
To evaluate \algrestless{} in the restless setting, we run the aforementioned algorithms on a problem with $K = 15$ arms over a time horizon of $T = 200,000$ rounds, setting $\epsilon = 1/4$. The payoff functions $\mu_i(\cdot)$ are chosen in these families: $F_{\text{exp}} = \left\{ f(t) = c \left( 1 - e^{-at} \right) \right\}$ and $F_\text{poly} = \left\{ f(t) = c \left(1 - b(t + b^{1/\rho})^{-\rho} \right) \right\}$, where $a, c, \rho \in (0,1]$ and $b \in \mathbb{R}_{\ge 0}$ are parameters, whose values have been selected randomly.
By construction all functions $f \in F_\text{exp} \cup F_\text{poly}$ satisfy Assumptions~\ref{ass:incr} and~\ref{ass:decrDeriv}. The functions coming from $F_{\text{exp}} $ (exponential functions) have a sudden increase, while ones from $F_\text{poly}$ (polynomial functions) have a slower growth rate, leading to different cumulative increments $\Upsilon_{\bm{\mu}}$. The stochasticity is realized by adding a Gaussian noise with $\sigma = 0.1$. The generated functions are shown in Figure~\ref{fig:15arms_rewards}.


The empirical cumulative regret $\widehat{R}_{\bm{\mu}}(\pi,t)$ is provided in Figure~\ref{fig:restless_15arms_regrets}. The results show that \texttt{SW-TS} is the algorithm that achieves the lowest regret at the horizon, even though its performance at the beginning is worse than the other algorithms. As commonly happening in practice, TS-based approaches tend to outperform UCB ones. Indeed, \algrestless{} displays the second-best curve overall and achieves the best performance among the UCB-like algorithms.

\subsection{Rested setting}
We employ the same arms generated for the restless case to evaluate \algrested{} in the rested setting. We plot the empirical cumulative regret in Figure~\ref{fig:rested_15arms_regrets}. \texttt{SW-TS} is confirmed as the best algorithm at the end of the time horizon, although other algorithms (\texttt{SW-UCB} and \texttt{SW-KL-UCB}) suffer less regret at the beginning of learning. \algrested pays the price of the initial exploration, but at the end of the horizon, it manages to achieve the second-best performance. Notice that, besides \algrested, all other baseline algorithms are designed for the restless setting and are not endowed with any guarantee on the regret in the rested scenario.

To highlight this fact, we designed a particular $2$-arms rising rested bandit in which the optimal arm reveals only when pulled a sufficient number of times (linear in $T$). The payoff functions, fulfilling Assumptions~\ref{ass:incr} and~\ref{ass:decrDeriv}, are shown in Figure~\ref{fig:2arms_rewards} and the algorithms' empirical regrets in Figure~\ref{fig:rested_2arms_regrets}. Note that in this setting the expected (instantaneous) regret may be negative for $t < \frac{19T}{400}$, and this is the case for most of the algorithms for $t < 20,000$. While for the first $\approx 20,000$ rounds \algrested{} is on par with the other algorithms, it outperforms all the other policies over a longer run. Note that the regret for \texttt{Rexp3} and \texttt{Ser4} is decreasing the slope for $t > 40,000$, meaning that they are somehow reacting to the change in the reward of the two arms. \texttt{SW-TS} starts reacting even later, at around $t \approx 100,000$. However, they are not prompt to detect such a change in the rewards and, therefore, collect a large regret in the first part of the learning process. The other algorithms suffer a linear regret at the end of the time horizon since they do not employ forgetting mechanisms or because the sliding window should be tuned knowing the characteristics of the expected reward.

\subsection{IMDB dataset (rested)} \label{sec:IMDB}
We investigate the performance of \algrested{} on an \emph{online model selection} task for a real-world dataset. We employ the IMDB dataset, made of $50,000$ reviews of movies (scores from $0$ to $10$). We preprocessed the data as done by~\citet{maas2011learning} to obtain a binary classification problem. Each review $\bm{x}_t$ lies in a $d = 10,000$ dimensional feature space, where each feature is the frequency of the most common English words. Each arm corresponds to a different online optimization algorithm, \ie two of them are Online Logistic Regression algorithms with different learning rate schemes, and the other five are Neural Networks with different topologies. We provide additional information on the arms of the bandit in Appendix~\ref{apx:imdb}.
At each round, a sample $\bm{x}_t$ is randomly selected from the dataset, a reward of $1$ is generated for a correct classification, $0$ otherwise, and, finally, the online update step is performed for the chosen algorithm. 

The empirical regret is plotted in Figure~\ref{fig:imdb}. We can see that \algrested{}, with $\epsilon = 1/32$ outperforms the considered baselines. Compared to the synthetic simulations, the smaller window choice is justified by the fact that we need to take into account that the average learning curves of the classification algorithms are not guaranteed to be non-decreasing nor concave on the single run. However, keeping the window linear in $N_{i,t-1}$ is crucial for the regret guarantees of Theorem~\ref{thr:theRegretRestedStochastic}.


\section{Discussion and Conclusions}
This paper studied the MAB problem when the payoffs are non-decreasing functions that evolve either when pulling the corresponding arm (rested) or for time passing (restless). We showed that, for the rested case, an assumption on the payoff (\eg concavity) is essential to make the problem learnable. We presented novel algorithms that suitably employ the concavity assumption to build proper estimators for both settings. These algorithms are proven to suffer a regret made of a first instance-independent component of $\widetilde{\mathcal{O}}(T^{\frac{2}{3}})$ and an instance-dependent component involving the cumulative increment function $\Upsilon_{\bm{\mu}}(\cdot,q)$. For the rested setting, ours represent the first no-regret algorithm for the stochastic rising bandits. The experimental evaluation confirmed our theoretical findings showing advantages over state-of-the-art algorithms designed for non-stationary bandits, especially in the rested setting. The natural future research direction consists of studying the complexity of the learning problem in stochastic rising rested and restless bandits, focusing on deriving suitable regret lower bounds. Other future works include investigating the best-arm identification setting and, motivated by the online model selection, analysing the alternative case in which the optimization algorithms associated with the arms act on a shared vector of parameters.

\section*{Acknowledgements}
The authors would like to Emmanuel Esposito, Saeed Masoudian, and Alessandro Montenegro for their contribution to the fix of the minor issues present in the manuscript.

\bibliography{biblio}
\bibliographystyle{icml2022}

\newpage
\appendix

\setlength{\abovedisplayskip}{12pt}
\setlength{\belowdisplayskip}{12pt}
\setlength{\textfloatsep}{20pt}

\onecolumn 

\section{Proofs and Derivations}
In this section, we provide the proof of the results presented in the main paper.

\subsection{Proofs of Section~\ref{sec:rresed}}\label{apx:prrested}

\thrRestedOptimal*
\begin{proof}
	The proof is reported in Proposition 1 of~\cite{heidari2016tight}.
\end{proof}

\begin{lemma}\label{lemma:nonLearnableLemma}
	In the noiseless ($\sigma = 0$) setting, there exists a $2$-armed non-decreasing non-concave bandit such that any learning policy $\pi$ suffers regret:
	\begin{align*}
		R_{\bm{\mu}}(\pi, T) \ge \left\lfloor \frac{T}{12} \right\rfloor.
	\end{align*}
\end{lemma}

\begin{proof}
Let $\bm{\mu}^A$ and $\bm{\mu}^B$ be two non-concave non-decreasing rested bandits, defined as:
\begin{align*}
	& \mu^A_1(n) = \mu^B_1(n) = \frac{1}{2},\\
	& \mu^A_2(n) = \begin{cases}
			0 & \text{if } n \le \lfloor \frac{T}{3} \rfloor \\ 
			1 & \text{otherwise}
		\end{cases} ,\\
		& \mu^B_2(n) = 0.
\end{align*}
It is clear that for $\bm{\mu}^A$ the optimal arm is $2$, whereas for  bandit $\bm{\mu}^B$ the optimal arm is $1$, having optimal performance respectively $J_{\bm{\mu}^A}^*(T) = \lceil \frac{2}{3}T \rceil$ and $J_{\bm{\mu}^B}^*(T) = \frac{T}{2}$.

\begin{center}
\begin{tikzpicture}[
  declare function={
    func(\x)= (\x < 3) * (0) + (\x >= 3) * (1);
  }
]
\begin{groupplot}[
  group style={
    group size=2 by 1,
    vertical sep=0pt,
    horizontal sep=4cm,
    group name=G},
     width=6cm,height=4cm,
  axis lines=middle,
  legend style={at={(1.2,1)},anchor=north,legend cell align=left} 
  ]
  \nextgroupplot[xmin=0, xmax=13, ymax=1.1, ymin=-0.02, domain=0:12, samples=1000, xtick = {0,3,12}, xticklabels = {$0$, $\lfloor\frac{T}{3}\rfloor$, $T$}, xlabel={$n$}]
  \addplot[red,  ultra thick] (x,.5);\addlegendentry{$\mu^A_1$}
  \addplot[blue,  ultra thick] {func(x)};\addlegendentry{$\mu^A_2$}
  
  \nextgroupplot[xmin=0, xmax=13, ymax=1.1, ymin=-0.02, domain=0:12, samples=1000, xtick = {0,12}, xticklabels = {$0$, $T$},xlabel={$n$}]
  \addplot[red,  ultra thick] (x,.5);\addlegendentry{$\mu^B_1$}
  \addplot[blue, ultra thick] (x,0);\addlegendentry{$\mu^B_2$}
\end{groupplot}
\end{tikzpicture}
\end{center}

 Let $\pi$ be an arbitrary policy. Since the learner will receive the same rewards for both bandits until at least $\lfloor \frac{T}{3} \rfloor$. Thus, we have:
\begin{align*}
	\pi(\mathcal{H}_t(\bm{\mu}^A)) = \pi(\mathcal{H}_t(\bm{\mu}^B)) \; \implies \; \E_{\bm{\mu}^A} \left[ N_{1,\lfloor \frac{T}{3} \rfloor} \right] = \E_{\bm{\mu}^B} \left[ N_{1,\lfloor \frac{T}{3} \rfloor} \right] =: n_1.
\end{align*}
Let us now compute the performance of policy $\pi$ in the two bandits and the corresponding regrets. Let us start with $\bm{\mu}^A$:
\begin{align}
	J_{\bm{\mu}^A}(\pi,T) & = \frac{1}{2}\E_{\bm{\mu}^A}[N_{1,T}] + \max \left\{0, \E_{\bm{\mu}^A}[N_{2,T}] - \left\lfloor \frac{T}{3} \right\rfloor \right\} \label{p:001}\\
	& = \frac{1}{2}\E_{\bm{\mu}^A}[N_{1,T}] + \max \left\{0, \left\lceil \frac{2}{3}T \right\rceil - \E_{\bm{\mu}^A}[N_{1,T}] \right\},\label{p:002}
\end{align}
where Equation~\eqref{p:001} follows from observing that we get reward from arm $2$ only if we pull it more than $\lfloor\frac{T}{3}\rfloor$ times and Equation~\eqref{p:002} derives from observing that $T = \E_{\bm{\mu}^A}[N_{1,T}] + \E_{\bm{\mu}^A}[N_{2,T}]$. Now, consider the two cases:
\paragraph{Case (i) : $\E_{\bm{\mu}^A}[N_{1,T}] \ge \lceil \frac{2}{3} T \rceil$}
\begin{align*}
	J_{\bm{\mu}^A}(\pi,T) = \frac{1}{2} \E_{\bm{\mu}^A}[N_{1,T}],
\end{align*}
that is maximized by taking $\E_{\bm{\mu}^A}[N_{1,T}] = T$. 
\paragraph{Case (ii) : $\E_{\bm{\mu}^A}[N_{1,T}] < \lceil \frac{2}{3} T\rceil$}
\begin{align*}
	J_{\bm{\mu}^A}(\pi,T) =\left\lceil \frac{2}{3}T \right\rceil - \frac{1}{2} \E_{\bm{\mu}^A}[N_{1,T}],
\end{align*}
that is maximized by taking the minimum value of $\E_{\bm{\mu}^A}[N_{1,T}]$ possible, that is $\E_{\bm{\mu}^A}[N_{1,T}] \ge \E_{\bm{\mu}^A}[N_{1,\lfloor\frac{T}{3} \rfloor}] = n_1$. Putting all together, we have:
\begin{align*}
	J_{\bm{\mu}^A}(\pi,T) \le \max \left\{ \frac{T}{2}, \left\lceil\frac{2}{3}T\right\rceil - \frac{n_1}{2} \right\} =  \left\lceil \frac{2}{3}T \right\rceil - \frac{n_1}{2},
\end{align*}
having observed that $n_1 \le \lfloor\frac{T}{3}\rfloor$. Let us now focus on the regret:
\begin{align*}
	R_{\bm{\mu}^A}(\pi, T) = J_{\bm{\mu}^A}^*(T) - J_{\bm{\mu}^A}(\pi,T)  =\left\lceil \frac{2}{3}T \right\rceil -\left\lceil \frac{2}{3}T \right\rceil + \frac{n_1}{2} =  \frac{n_1}{2}.
\end{align*}
Consider now bandit $\bm{\mu}^B$, we have:
\begin{align*}
	J_{\bm{\mu}^B}(\pi,T) =  \frac{1}{2}  \E_{\bm{\mu}^B}[N_{1,T}] \le \frac{n_1}{2} + \left\lfloor \frac{T}{3} \right\rfloor,
\end{align*}
having observed that $ \E_{\bm{\mu}^B}[N_{1,T}] =  n_1 + \E_{\bm{\mu}^B}[N_{1,T}] -  \E_{\bm{\mu}^B}[N_{1,\lfloor \frac{T}{3} \rfloor}] \le n_1 + \lceil \frac{2}{3}T \rceil$. Let us now compute the regret:
\begin{align*}
	R_{\bm{\mu}^B}(\pi, T) =  J_{\bm{\mu}^B}^*(T) - J_{\bm{\mu}^B}(\pi,T) = \frac{T}{2} -  \frac{n_1}{2} - \left\lfloor\frac{T}{3}\right\rfloor = \left\lceil\frac{T}{6}\right\rceil - \frac{n_1}{2}.
\end{align*}
Finally, the worst-case regret can be lower bounded as follows:
\begin{align*}
	\inf_{\pi} \sup_{\bm{\mu}}  R_{\bm{\mu}}(\pi,T) \ge \inf_{\pi} \max \left\{R_{\bm{\mu}^A}(\pi, T), R_{\bm{\mu}^B}(\pi, T) \right\} =\inf_{n_1 \in \left[ 0,\lfloor \frac{T}{3} \rfloor\right]} \max \left\{  \frac{n_1}{2} , \left\lceil \frac{T}{6} \right\rceil - \frac{n_1}{2}\right\} \ge \frac{1}{2} \left\lceil \frac{T}{6} \right\rceil \ge \left\lfloor \frac{T}{12} \right\rfloor,
\end{align*}
having minimized over $n_1$.
\end{proof}

\nonLearnable*
\begin{proof}
	It is sufficient to rescale the mean function of the proof of Lemma~\ref{lemma:nonLearnableLemma} by the quantity $\gamma_{\max}$.
\end{proof}

\begin{restatable}[]{lemma}{lemmaDetRested}\label{lemma:lemmaDetRested}
	For every arm $i\in [K]$ and every round $t \in [T]$, let us define:
	\begin{align*}
		\overline{\mu}_i^{\text{\red}}(t) \coloneqq \mu_i(N_{i,t-1}) + (t - N_{i,t-1}) \gamma_i(N_{i,t-1}-1),
	\end{align*}
	if $N_{i,t-1} \ge 2$ else $\overline{\mu}_i^{\text{\red}}(t) \coloneqq +\infty$. Then, $\overline{\mu}_i^{\text{\red}}(t) \ge \mu_i(t) $	and, if $N_{i,t-1} \ge 2$, it holds that:
	\begin{align*}
		\overline{\mu}_i^{\text{\red}}(t) - \mu_i(N_{i,t}) \le (t-N_{i,t-1})  \gamma_i(N_{i,t-1}-1).
	\end{align*}
\end{restatable}
\begin{proof}
	Let us consider the following derivation:
	\begin{align*}
		\mu_i(t) = \mu_i(N_{i,t-1}) + \sum_{n=N_{i,t-1}}^{t-1} \gamma_i(n) \le \mu_i(N_{i,t-1}) + (t-N_{i,t-1}) \gamma_i(N_{i,t-1}-1) \eqqcolon \overline{\mu}^{\text{\red}}_i(t),
	\end{align*}
	where the inequality holds thanks to Assumption~\ref{ass:decrDeriv}, having observed that $\sum_{n=N_{i,t-1}}^{t-1} \gamma_i(n) \le (t-N_{i,t-1}) \gamma_i(N_{i,t-1}) \le (t-N_{i,t-1}) \gamma_i(N_{i,t-1}-1)$. 
	For the bias bound, when $N_{i,t-1} \ge 2$, we consider the following derivation:
	\begin{align*}
		\overline{\mu}^{\text{\red}}_i(t) - \mu_i(N_{i,t}) & = \mu_i(N_{i,t-1}) + (t-N_{i,t-1}) \gamma_i(N_{i,t-1}-1) - \mu_i(N_{i,t}) \le (t-N_{i,t-1}) \gamma_i(N_{i,t-1}-1).
	\end{align*}
	having observed that $ \mu_i(N_{i,t-1})\le  \mu_i(N_{i,t})$ by Assumption~\ref{ass:incr}.
\end{proof}

\theRegretDetRested*
\begin{proof}
We have to analyze the following expression:
\begin{align*}
	R_{\bm{\mu}}(\text{\algrested},T) = \sum_{t=1}^T \mu_{i^*}(t) - \mu_{I_t}(N_{i,t}) ,
\end{align*}
where $i^* \in \argmax_{i \in [K]}\left\{ \sum_{l \in [T]} \mu_i(l) \right\}$. We consider a term at a time, use $B_i(t)\equiv\overline{\mu}_i^{\text{\red}}(t)$, and we exploit the optimism, \ie $B_{i^*}(t) \le B_{I_t}(t)$:
\begin{align*}
	\mu_{i^*}(t) - \mu_{I_t}(N_{I_t,t}) + B_{I_t}(t) - B_{I_t}(t) & \le  \min \left\{ 1, \underbrace{\mu_{i^*}(t) - B_{i^*}(t)}_{\le 0} + B_{I_t}(t) - \mu_{I_t}(N_{I_t,t}) \right\}\\
	&  \le \min \left\{ 1, B_{I_t}(t) - \mu_{I_t}(N_{I_t,t})\right\}.
\end{align*}
Now we work on the term inside the minimum when $N_{I_t,t-1} \ge 2$:
\begin{align*}
B_{I_t}(t) - \mu_{I_t}(N_{I_t,t}) & = \overline{\mu}_i^{\text{\red}}(t) - \mu_{I_t}(N_{I_t,t})
\le (t - N_{i,t-1}) \gamma_{I_t}( N_{i,t-1}-1),
\end{align*}
where the inequality follows from Lemma~\ref{lemma:lemmaDetRested}.
We are going to decompose the summation of this term over the $K$ arms:
\begin{align*}
R_{\bm{\mu}}(\text{\algrested},T) & \le \sum_{t=1}^T  \min \left\{ 1, (t - N_{i,t-1}) \gamma_{I_t}( N_{i,t-1}-1)\right\} \\
& \le 2K + \sum_{i \in [K]} \sum_{j=3}^{N_{i,T}} \min \left\{ 1,(t_{i,j} - (j-1)) \gamma_{i}( j-2)\right\},
\end{align*}
where $t_{i,j} \in [T]$ is the round at which arm $i \in [K]$ was pulled for the $j$-th time.
Now, $q \in [0,1]$, then for any $x \ge 0$ it holds that $\min\{1,x\} \le \min\{1,x\}^{q} \le x^{q}$. By applying this latter inequality to the inner summation, we get:
\begin{align*}
\sum_{j=3}^{N_{i,T}} \min \left\{ 1, (t_{i,j}-(j-1)) \gamma_i(j-2)  \right\} & \le \sum_{j=3}^{N_{i,T}} \min \left\{1, T \gamma_i(j-2)\right\} \le T^q \sum_{j=3}^{N_{i,T}} \gamma_i(j-2)^q,
\end{align*}
having used $t_{i,j}-(j-1) \le T$.
Summing over the arms, we obtain:
\begin{align*}
	 T^q \sum_{i \in [K]}  \sum_{j=3}^{N_{i,T}} \gamma_i(j-2)^q \le T^q K \Upsilon_{\bm{\mu}}\left( \left\lceil \frac{T}{K} \right\rceil,q \right),
\end{align*}
where the last inequality is obtained from Lemma~\ref{lemma:boundUpsilonMax}.
\end{proof}

\paragraph{Estimator Construction for the Stochastic Rising Rested  Setting}
Before moving to the proofs, we provide some intuition behind the estimator construction. We start observing that for every $l\in \{2, \dots, N_{i,t-1}\}$, we have that:
\begin{align*}
	\mu_i(t) = \textcolor{color11}{\underbrace{\mu_i(l)}_{\text{(past payoff)}}} + \textcolor{color12}{\underbrace{\sum_{j=l}^{t-1} \gamma_i(j)}_{\text{(sum of future increments)}}} \le  \textcolor{color11}{\underbrace{\mu_i(l)}_{\text{(past payoff)}}} + \textcolor{color12}{(t-l)\underbrace{\gamma_i(l-1)}_{\text{(past increment)}}},
\end{align*}
where the inequality follows from Assumption~\ref{ass:decrDeriv}.\footnote{The estimator of the deterministic case in Equation~\eqref{eq:estRestedDetEst} is obtained by setting $l=N_{i,t-1}$.} Since we do not have access to the exact payoffs $\mu_i(l)$ and exact increments $\gamma_i(l-1) = \mu_i(l) - \mu_i(l-1)$, one may be tempted to directly replace them with the corresponding point estimates $R_{t_{i,l}}$ and $R_{t_{i,l}} - R_{t_{i,l-1}}$ and average the resulting estimators for a window of the most recent $h$ values of $l$. Unfortunately, while replacing $\mu_i(l)$ with $R_{t_{i,l}}$ is a viable option, replacing $\gamma_i(l-1)$ with $R_{t_{i,l}} - R_{t_{i,l-1}}$ will prevent concentration since the estimate  $R_{i,t_l} - R_{i,t_{l-1}}$ is too unstable. To this end, before moving to the estimator, we need a further bounding step to get a more stable, although looser, quantity. Based on Lemma~\ref{lemma:boundDeriv}, we bound for every $l\in \{2, \dots, N_{i,t-1}\}$ and $h \in[ l -1]$:
\begin{align*}
\textcolor{color12}{\underbrace{\gamma_i(l-1)}_{\text{(past increment at $l$)}}} \le \textcolor{color12}{\underbrace{\frac{\mu_i(l) - \mu_i(l-h)}{h}}_{\text{(average past increment over $\{l-h,\dots,l\}$)}}}.
\end{align*}
We can now introduce the optimistic approximation of $\mu_i(t)$, \ie $\widetilde{\mu}_i^{\text{\red},h}(t)$, and the corresponding estimator, \ie $\widehat{\mu}_i^{\text{\red},h}(t)$, that are defined in terms of a window of size $1 \le h \le \lfloor N_{i,t-1}/2 \rfloor$:
\begin{align*}
&  \widetilde{\mu}_i^{\text{\red},h} (t)  \coloneqq \frac{1}{h} \sum_{l=N_{i,t-1}-h+1}^{N_{i,t-1}} \Bigg(\textcolor{color11}{\underbrace{\mu_i(l)}_{\text{(past payoff)}}} + \textcolor{color12}{(t-l) \underbrace{\frac{\mu_i(l) - \mu_{i}(l-h)}{h}}_{\text{(average past increment)}}}\Bigg),\\
& \widehat{\mu}_i^{\text{\red},h} (t)  \coloneqq \frac{1}{h} \sum_{l=N_{i,t-1}-h+1}^{N_{i,t-1}} \Bigg(\textcolor{color11}{\underbrace{R_{t_{i,l}}}_{\text{(estimated past payoff)}}} + \textcolor{color12}{(t-l) \underbrace{\frac{R_{t_{i,l}} - R_{t_{i,l-h}}}{h}}_{\text{(estimated average past increment)}}}\Bigg).
\end{align*}

The proof is composed of the following steps:
\begin{enumerate}[label=(\roman*)]
	\item Lemma~\ref{lemma:firstEstimator} shows that $\widetilde{\mu}_i^{\text{\red},h}(t)$ is an upper-bound for $\mu_i(t)$ and provides a bound to its bias \wrt $\mu_i(N_{i,t})$ for every value of $h$;
	\item Lemma~\ref{lemma:firstEstimatorConc} analyzes the concentration of $\widehat{\mu}_i^{\text{\red},h}(t)$ around $\widetilde{\mu}_i^{\text{\red},h}(t)$ for a specific choice of $\delta_t = t^{-\alpha}$ and when $h_{i,h} \coloneqq h(N_{i,t-1})$ is a function of the number of pulls $N_{i,t-1}$ only;
	\item Theorem~\ref{thr:theRegretRestedStochastic} bounds the expected regret of \algrested when $h_{i,h} =  \lfloor \epsilon N_{i,t-1} \rfloor$, for $\epsilon \in (0,1/2)$.
\end{enumerate}

\begin{restatable}[]{lemma}{firstEstimator}\label{lemma:firstEstimator}
For every arm $i \in [K]$, every round $t \in [T]$, and window width $ 1 \le h \le \lfloor N_{i,t-1}/2 \rfloor$, let us define:
\begin{align*}
	 \widetilde{\mu}_i^{\text{\red},h} (t)  \coloneqq \frac{1}{h} \sum_{l=N_{i,t-1}-h+1}^{N_{i,t-1}} \left(\mu_i(l) + (t-l) \frac{\mu_i(l) - \mu_{i}(l-h)}{h}\right),
\end{align*}
otherwise if $h=0$, we set $\widetilde{\mu}_i^{\text{\red},h}(t)  \coloneqq +\infty$. Then, $\widetilde{\mu}_i^{\text{\red},h}(t) \ge \mu_i(t)$ and, if $N_{i,t-1} \ge 2$, it holds that:
\begin{align*}
	& \widetilde{\mu}_i^{\text{\red},h}(t)   - \mu_i(N_{i,t}) \le \frac{1}{2}(2t - 2N_{i,t-1} + h -1) \gamma_i(N_{i,t-1}-2h+1).
\end{align*}
\end{restatable}
\begin{proof}
Following the derivation provided above, we have for every $l\in \{2, \dots, N_{i,t-1}\}$:
\begin{align}
\mu_i(t) & = {\mu_i(l)} + \sum_{j=l}^{t-1} \gamma_i(j) \notag\\
&  \le  \mu_i(l) + (t-l){\gamma_i(l-1)} \label{:pp:-001}\\
& \le  \mu_i(l) + (t-l)\frac{\mu_i(l) - \mu_i(l-h)}{h}, \label{:pp:-002}
\end{align}
where line~\eqref{:pp:-001} follows from Assumption~\ref{ass:decrDeriv}, line~\eqref{:pp:-002} is obtained from Lemma~\ref{lemma:boundDeriv}. By averaging over the most recent $1 \le h \le \lfloor N_{i,t-1}/2 \rfloor$ pulls, we obtain:
\begin{align*}
	\mu_i(t) &\le \frac{1}{h} \sum_{l=N_{i,t-1}-h+1}^{N_{i,t-1}} \left(\mu_i(l) + (t-l)\frac{\mu_i(l) - \mu_i(l-h)}{h} \right) \eqqcolon \widetilde{\mu}_i^{\text{\red},h}(t).
\end{align*}
For the bias bound, when $N_{i,t-1} 	\ge 2$, we have:
\begin{align}
\widetilde{\mu}_i^{\text{\red},h}(t) - \mu_i(N_{i,t}) & = \frac{1}{h} \sum_{l=N_{i,t-1}-h+1}^{N_{i,t-1}}  \left(\mu_i(l) +  (t-l)  \frac{\mu_i(l) - \mu_i(l-h)}{h} \right) - \mu_i(N_{i,t}) \label{:pp:-0032}\\
& \le \frac{1}{h} \sum_{l=N_{i,t-1}-h+1}^{N_{i,t-1}} (t-l)  \frac{\mu_i(l) - \mu_i(l-h)}{h} \notag\\
& = \frac{1}{h} \sum_{l=N_{i,t-1}-h+1}^{N_{i,t-1}} (t-l)  \frac{1}{h} \sum_{j=l-h}^{l-1} \gamma_j(l)\notag \\
& \le \frac{1}{h} \sum_{l=N_{i,t-1}-h+1}^{N_{i,t-1}} (t-l) \gamma_i(l-h) \label{:pp:-003}\\
& \le \frac{1}{2}(2t - 2N_{i,t-1} + h -1) \gamma_i(N_{i,t-1}-2h+1) \label{:pp:-004}.
\end{align}
where line~\eqref{:pp:-0032} follows from Assumption~\ref{ass:incr} applied as $\mu_i(l) \le \mu_i(N_{i,t})$, line~\eqref{:pp:-003} follows from Assumption~\ref{ass:decrDeriv} and bounding $ \frac{1}{h} \sum_{j=l-h}^{l-1} \gamma_j(l) \le \gamma_i(l-h)$ and line~\eqref{:pp:-004} is derived still from Assumption~\ref{ass:decrDeriv}, $\gamma_i(l-h) \le \gamma_i(N_{i,t-1}-2h+1)$ and computing the summation.
\end{proof}
\begin{restatable}[]{lemma}{firstEstimatorConc}\label{lemma:firstEstimatorConc}
For every arm $i \in [K]$, every round $t \in [T]$, and window width $1 \le h \le \lfloor N_{i,t-1} / 2\rfloor$, let us define:
\begin{align*}
	 & \widehat{\mu}_i^{\text{\red},h} (t)  \coloneqq \frac{1}{h} \sum_{l=N_{i,t-1}-h+1}^{N_{i,t-1}} \left(R_{t_{i,l}} + (t-l) \frac{R_{t_{i,l}} - R_{t_{i,l-h}}}{h}\right),\\
	 & \beta^{\text{\red},h}_i(t,\delta)\coloneqq \sigma  (t-N_{i,t-1}+h-1) \sqrt{ \frac{ 10  \log \frac{1}{\delta} }{h^3} },
\end{align*}
otherwise if $h=0$, we set $\widehat{\mu}_i^{\text{\red},h} (t) \coloneqq +\infty$ and $\beta^{\text{\red},h}_i(t,\delta)\coloneqq+\infty$ .
Then, if the window size depends on the number of pulls only $h_{i,t} = h(N_{i,t-1}) $ and if $\delta_t = t^{-\alpha}$ for some $\alpha > 2$, it holds for every round  $t \in [T]$ that:
\begin{align*}
	\Pr \left(\left| \widehat{\mu}_i^{\text{\red},h_{i,t}}(t) - \widetilde{\mu}_i^{\text{\red},h_{i,t}}(t) \right| >  \beta^{\text{\red},h_{i,t}}_i(t,\delta_t) \right) \le 2t^{1-\alpha}.
\end{align*}
\end{restatable}

\begin{proof}
First of all, we observe under the event $\{h_{i,t} = 0\}$, then $\widehat{\mu}_i^{\text{\red},h_{i,t}}(t) = \widetilde{\mu}_i^{\text{\red},h_{i,t}}(t) = \beta^{\text{\red},h_{i,t}}_i(t,\delta_t)= +\infty$. By convening that $(+\infty) - (+\infty) = 0$, we have that $0 >  \beta^{\text{\red},h_{i,t}}_i(t,\delta_t)$ is not satisfied. Thus, we perform the analysis under the event $\{h_{i,t} \ge 1\}$. We first get rid of the dependence on the random number of pulls $N_{i,t-1}$:
%
\begin{align}
 \Pr & \left(\left| \widehat{\mu}_i^{\text{\red},h_{i,t}}(t) - \widetilde{\mu}_i^{\text{\red},h_{i,t}}(t) \right| >  \beta^{\text{\red}, h_{i,t}}_i(t,\delta_t) \right) \notag \\
 & \qquad\qquad = \Pr \left(\left| \widehat{\mu}_i^{\text{\red},h(N_{i,t-1}) }(t) - \widetilde{\mu}_i^{\text{\red},h(N_{i,t-1}) }(t) \right| >  \beta^{\text{\red},h(N_{i,t-1}) }_i(t,\delta_t) \right) \label{:pp:-000} \\
 & \qquad\qquad \le \Pr \left(\exists n \in \{0,\dots,t-1\} \text{ s.t. } h(n) \ge 1\,:\, \left| \widehat{\mu}_i^{\text{\red},h(n)}(t) - \widetilde{\mu}_i^{\text{\red}, h(n)}(t) \right| >  \beta^{\text{\red},h(n)}_i(t,\delta_t) \right) \notag \\
 & \qquad\qquad \le \sum_{n \in \{0,\dots,t-1\} \,:\, h(n) \ge 1} \Pr \left(\left| \widehat{\mu}_i^{\text{\red},h(n)}(t) - \widetilde{\mu}_i^{\text{\red},h(n)}(t) \right| >  \beta^{\text{\red},h(n)}_i(t,\delta_t) \right), \label{:pp:-010}
\end{align}
where line~\eqref{:pp:-000} derives from the definition of $h_{i,t} = h(N_{i,t-1})$ and  line~\eqref{:pp:-010} follows from a union bound over the possible values of $N_{i,t-1}$. Now, having fixed the value of $n$, we rewrite the quantity to be bounded:
\begin{align*}
h(n)  \left( \widehat{\mu}_i^{\text{\red},h(n)}(t) - \widetilde{\mu}_i^{\text{\red},h(n)}(t) \right) & = \sum_{l=n-h(n)+1}^{n} \left(X_l + (t-l) \frac{X_l - X_{l-h(n)}}{h(n)}\right)\\
&  = \sum_{l=n-h(n)+1}^{n}  \left(1 + \frac{t-l}{h(n)}\right) X_l - \sum_{l=n-h(n)+1}^{n}  \frac{t-l}{h(n)} \cdot X_{l-h(n)},
\end{align*}
where $X_{l} \coloneqq R_{t_{i,l}} - \mu_i(l)$. It is worth noting that we can index $X_l$ with the number of pulls $l$ only as the distribution of $R_{t_{i,l}}$ is fully determined by $l$ and $n$  (that are non-random quantities now) and, consequently, all variables $X_l$ and $X_{l-h(n)}$ are independent.

Now we apply Azuma-Ho\"effding's inequality of Lemma~\ref{lemma:hoeffding} for weighted sums of subgaussian martingale difference sequences. To this purpose, we compute the sum of the square weights:
\begin{align}
	\sum_{l=n-h(n)+1}^{n} \left(1+ \frac{t-l}{h(n)}\right)^2 & + \sum_{l=n-h(n)+1}^{n} \left( \frac{t-l}{h(n)} \right)^2 \nonumber\\
	 & \le h(n) \left(1+ \frac{ t-n+h(n)-1}{h(n)}\right)^2 + h(n)\left( \frac{ t-n+h(n)-1}{h(n)} \right)^2 \label{eq:prim}\\
	 &  \le \frac{ 5(t-n+h(n)-1)^2}{h(n)}, \label{eq:secon}
	 \end{align}
	 where line~\eqref{eq:prim} follows from bounding $t-l \le t-n+h(n)-1$ and line~\eqref{eq:secon} from observing that $\frac{ t-n+h(n)-1}{h(n)} \ge 1$. Thus, we have:
	 \begin{align*}
	 & \Pr \left(\left| \widehat{\mu}_i^{\text{\red},h(n)}(t) - \widetilde{\mu}_i^{\text{\red}, h(n)}(t) \right| >  \beta^{\text{\red},h(n)}_i(t,\delta_t) \right) \\
	 & \qquad\qquad \le \Pr \left(\left| \sum_{l=n-h(n)+1}^{n}  \left(1 + \frac{t-l}{h(n)}\right) X_l - \sum_{l=n-h(n)+1}^{n}  \frac{t-l}{h(n)} \cdot X_{l-h(n)} \right| > h(n) \beta^{\text{\red},h(n)}_i(t,\delta_t) \right)\\
	 & \qquad\qquad 2 \exp \left( - \frac{\left(h(n) \beta^{\text{\red},h(n)}_i(t,\delta_t) \right)^2}{2 \sigma^2 \left( \frac{ 5(t-n+h(n)-1)^2}{h(n)} \right)}\right) = 2 \delta_t.
	 \end{align*}
	 By replacing this result into Equation~\eqref{:pp:-010}, and recalling the value of $\delta_t$, we obtain:
	 \begin{align*}
	 \sum_{n\in\{0,\dots,t-1\} \,:\, h(n) \ge 1}  2 \delta_t \le \sum_{n=0}^{t-1}  2 \delta_t = \sum_{n=0}^{t-1} 2t^{-\alpha} \le 2t^{1-\alpha}.
	 \end{align*}
\end{proof}

\theRegretRestedStochastic*
\begin{proof}
Let us define the good events $\mathcal{E}_t = \bigcap_{i \in [K]} \mathcal{E}_{i,t}$ that correspond to the event in which all confidence intervals hold:
\begin{align*} 
\mathcal{E}_{i,t} \coloneqq \left\{ \left| \widetilde{\mu}_i^{\text{\red},h_{i,t}}(t) - \widehat{\mu}_i^{\text{\red},h_{i,t}}(t)\right| \le \beta^{\text{\red},h_{i,t}}_i(t)\right\} \qquad \forall i \in [T], \,i \in [K]
\end{align*}
We have to analyze the following expression:
\begin{align*}
	R_{\bm{\mu}}(\text{\algrested},T) = \E \left[\sum_{t=1}^T \mu_{i^*}(t) - \mu_{I_t}(N_{i,t})\right] ,
\end{align*}
where $i^* \in \argmax_{i \in [K]}\left\{\sum_{l \in  [T]} \mu_i(l) \right\}$.  We decompose the above expression according to the good events $\mathcal{E}_t$:
\begin{align}
R_{\bm{\mu}}(\text{\algrested},T) & = \sum_{t=1}^T \mathbb{E} \left[\left( \mu_{i^*}(t) - \mu_{I_t}(N_{I_t,t}) \right)\mathds{1}\{\mathcal{E}_t\} \right] + \sum_{t=1}^T \mathbb{E} \left[\left( \mu_{i^*}(t) - \mu_{I_t}(N_{I_t,t}) \right)\mathds{1}\{\lnot\mathcal{E}_t\}\right] \\
& \le \sum_{t=1}^T \mathbb{E} \left[\left( \mu_{i^*}(t) - \mu_{I_t}(N_{I_t,t}) \right)\mathds{1}\{\mathcal{E}_t\}\right] + \sum_{t=1}^T \mathbb{E} \left[\mathds{1}\{\lnot\mathcal{E}_t\}\right], \label{eq:terz}
\end{align}
where we exploited $\mu_{i^*}(t) - \mu_{I_t}(N_{I_t,t}) \le 1$ in line~\eqref{eq:terz}.
Now, we bound the second summation, recalling that $\alpha > 2$:
\begin{align*}
\sum_{t=1}^T \mathbb{E} \left[\mathds{1}\{\lnot\mathcal{E}_t\}\right] = \sum_{t=1}^T \Pr\left(\lnot\mathcal{E}_t\right) \le  1 + \sum_{t=2}^T \Pr\left( \lnot \bigcap_{i \in [K]}  \mathcal{E}_{i,t}\right) =1 + \sum_{t=2}^T \Pr\left( \bigcup_{i \in [K]} \lnot \mathcal{E}_{i,t}\right) \le 1 + \sum_{i \in [K]} \sum_{t=2}^T \Pr\left( \lnot \mathcal{E}_{i,t}\right),
\end{align*}
where the first inequality is obtained with $\Pr(\lnot \mathcal{E}_{1}) \le 1$ and the second with a union bound over $[K]$. Recalling $\Pr(\lnot\mathcal{E}_{i,t})$ was bounded in Lemma~\ref{lemma:firstEstimatorConc}, we bound the summation with the integral as in Lemma~\ref{lemma:tecSum} to get:
\begin{align*}
	\sum_{i \in [K]} \sum_{t=2}^T \Pr\left( \lnot \mathcal{E}_{i,t}\right) \le \sum_{i \in [K]} \sum_{t=2}^T 2t^{1-\alpha} \le 2K \int_{x=1}^{+\infty} x^{1-\alpha} \de x = \frac{2K}{\alpha-2}.
\end{align*}
From now on, we proceed the analysis under the good events $\mathcal{E}_t$, recalling that $B_i(t) \equiv \widehat{\mu}_i^{\text{\red}, h_{i,t}}(t) + \beta_i^{\text{\red}, h_{i,t}}(t,\delta_t)$. 
We consider each addendum of the summation and we exploit the optimism, \ie $B_{i^*}(t) \le B_{I_t}(t)$:
\begin{align*}
	\mu_{i^*}(t) - \mu_{I_t}(N_{I_t,t}) + B_{I_t}(t) - B_{I_t}(t) & \le  \min \left\{ 1, \underbrace{\mu_{i^*}(t) - B_{i^*}(t)}_{\le 0} + B_{I_t}(t) - \mu_{I_t}(N_{I_t,t}) \right\} \\
	& \le \min \left\{ 1, B_{I_t}(t) - \mu_{I_t}(N_{I_t,t})\right\}.
\end{align*}
Now, we work on the term inside the minimum:
\begin{align}
	B_{I_t}(t) - \mu_{I_t}(N_{I_t,t}) &= \widehat{\mu}_{I_t}^{\text{\red},h_{I_t,t}}(t)  + \beta^{\text{\red},h_{I_t,t}}_{I_t}(t,\delta_t)-\mu_{I_t}(N_{I_t,t}) \label{:pp:-2001} \\
	& \le \underbrace{\widetilde{\mu}_{I_t}^{\text{\red},h_{I_t,t}}(t) - \mu_{I_t}(N_{I_t,t})}_{\text{(a)}}+ \underbrace{2 \beta^{\text{\red},h_{I_t,t}}_{I_t}(t,\delta_t)}_{\text{(b)}},\label{:pp:-2002}
\end{align}	
where line~\eqref{:pp:-2001} follows from the definition of $B_{i}(t)$, and line~\eqref{:pp:-2002} derives from the fact that we are under the good event $\mathcal{E}_t$. We now decompose over the arms and consider one term at a time. We start with (a):
\begin{align}
\sum_{t=1}^T \min\left\{1, \widetilde{\mu}_{I_t}^{\text{\red},h_{I_t,t}}(t) - \mu_{I_t}(N_{I_t,t}) \right\} & \le 2K + \sum_{i \in [K]} \sum_{j=3}^{N_{i,T}}\min \left\{1, \widetilde{\mu}_{i}^{\text{\red},h_{i,t_{i,j}}}(t_{i,j}) - \mu_{i}(j) \right\} \notag\\
& \le 2K + \sum_{i \in [K]} \sum_{j=3}^{N_{i,T}}\min \left\{1, \frac{1}{2} (2t_{i,j} - 2(j-1)+ h_{i,t_{i,j}} - 1)\gamma_{i}((j-1)-2h_{i,t_{i,j}}+1) \right\}  \label{:pr:001}\\
& \le 2K + \sum_{i \in [K]} \sum_{j=3}^{N_{i,T}}\min \left\{1, T \gamma_{i}(j- 2\lfloor \epsilon (j-1) \rfloor) \right\}\label{:pr:002}\\
& \le 2K + \sum_{i \in [K]} \sum_{j=3}^{N_{i,T}}\min \left\{1, T \gamma_{i}(\lfloor (1-2\epsilon) j \rfloor ) \right\}\label{:pr:003}\\
& \le 2K + T^q \sum_{i \in [K]} \sum_{j=3}^{N_{i,T}} \gamma_{i}(\lfloor (1-2\epsilon) j \rfloor )^q \label{:pr:0035}\\
& \le 2K + T^q \left\lceil \frac{1}{1-2\epsilon}  \right\rceil  \sum_{i \in [K]} \sum_{j=\lfloor 3(1-2\epsilon)\rfloor}^{\lfloor (1-2\epsilon) N_{i,T} \rfloor} \gamma_{i}( j ) \label{:pr:004}\\
& \le 2K + K T^q \left\lceil \frac{1}{1-2\epsilon}  \right\rceil  \Upsilon_{\bm{\mu}}\left( \left\lceil (1-2\epsilon) \frac{T}{K} \right\rceil,q\right),\label{:pr:007}
\end{align}
where line~\eqref{:pr:001} follows from Lemma~\ref{lemma:firstEstimator}, line~\eqref{:pr:002} is obtained by bounding $2t_{i,j} - 2(j-1)+ h_{i,t_{i,j}} - 1 \le 2T$ and exploiting the definition of $h_{i,t} = \lfloor \epsilon N_{i,t-1} \rfloor$, line~\eqref{:pr:003} follows from the observation $j- 2\lfloor \epsilon (j-1) \rfloor  \ge j- 2\epsilon (j-1)  \ge \lfloor (1-2\epsilon) j \rfloor$, line~\eqref{:pr:0035} is obtained from the already exploited inequality $\min\{1,x\} \le \min\{1,x\}^q \le x^q$ for $q \in [0,1]$, line~\eqref{:pr:004} is an application of Lemma~\ref{lemma:sumWithFloor}, and line~\eqref{:pr:007} follows from Lemma~\ref{lemma:boundUpsilonMax} recalling that $\sum_{i\in [K]} \lfloor  (1-2\epsilon) N_{i,T} \rfloor \le (1-2\epsilon) T$.

Let us now move to the concentration term (b). We decompose over the arms as well, taking care of the pulls in which $h_{i,j} = 0$, that are at most $1+\left\lceil \frac{1}{\epsilon} \right\rceil$:
\begin{align}
	\sum_{t=1}^T \min\left\{1, 2 \beta^{\text{\red},h_{I_t,t}}_{I_t}(t,\delta_t)\right\} & \le K + K\left\lceil \frac{1}{\epsilon} \right\rceil + \sum_{i \in [K]} \sum_{j=\left\lceil \frac{1}{\epsilon} \right\rceil + 1}^{N_{i,T}}\min \left\{1, 2\sigma (t_{i,t}- (j-1) + h_{i,t_{i,t}} - 1) \sqrt{\frac{10 \log (t^\alpha)}{h_{i,t_{i,t}}^3}} \right\} \notag.\\
	& = K + K\left\lceil \frac{1}{\epsilon} \right\rceil + \sum_{i \in [K]} \sum_{j=\left\lceil \frac{1}{\epsilon} \right\rceil + 1}^{N_{i,T}}\min \left\{1, 2\sigma T \sqrt{\frac{10 \alpha \log (T)}{ \lfloor \epsilon (j-1) \rfloor^3}} \right\},\label{:pr:008}
%
%
%
\end{align}
where line~\eqref{:pr:008} follows from bounding $t^\alpha \le T^\alpha$ and from the definition of $h_{i,t} = \lfloor \epsilon N_{i,t-1} \rfloor$. To bound the summation, we compute the minimum integer value $j^* $ (that turns out to be independent of $i$) of $j$ such that the minimum is attained by its second argument:
\begin{align*}
	2\sigma T \sqrt{\frac{10 \alpha \log (T)}{ \lfloor \epsilon (j-1) \rfloor^3}}  \le 1 & \implies \lfloor \epsilon (j-1) \rfloor \ge (2\sigma T)^{\frac{2}{3}} \left( 10 \alpha \log T \right)^{\frac{1}{3}} \\
	& \implies j^* = \left\lceil \frac{1+\epsilon + (2\sigma T)^{\frac{2}{3}} \left( 10 \alpha \log T \right)^{\frac{1}{3}}}{\epsilon} \right\rceil.
\end{align*}
Thus, we have:
\begin{align}
 K + K\left\lceil \frac{1}{\epsilon} \right\rceil + \sum_{i \in [K]} \sum_{j=\left\lceil \frac{1}{\epsilon} \right\rceil + 1}^{N_{i,T}}& \min \left\{1, 2\sigma T \sqrt{\frac{10 \alpha \log (T)}{ \lfloor \epsilon (j-1) \rfloor^3}} \right\}  \le  K + K\left\lceil \frac{1}{\epsilon} \right\rceil + \sum_{i \in [K]} \left(\sum_{j=\lceil \frac{1}{\epsilon}\rceil+1}^{j^*} 1 + \sum_{j=j^*+1}^{N_{i,T}} 2\sigma T \sqrt{\frac{10 \alpha \log (T)}{ \lfloor \epsilon (j-1) \rfloor^3}} \right) \label{:pr:991}\\
& \le K + K\left\lceil \frac{1}{\epsilon} \right\rceil + K\left(j^*-1-\left\lceil \frac{1}{\epsilon} \right\rceil +1 \right) + 2 K \sigma T \sqrt{{10 \alpha \log (T)}} \int_{x=j^*}^{+\infty}  \frac{1}{(\epsilon (x-1) - 1)^\frac{3}{2}} \de x \label{:pr:992} \\
& = K + Kj^* + \frac{4 K \sigma T \sqrt{{10 \alpha \log (T)}}}{ \epsilon \left( \epsilon (j^*-1) - 1 \right)^{\frac{1}{2}}} \notag \\
& = K \left(3 + \frac{1}{\epsilon} \right)+ \frac{3K}{\epsilon}(2\sigma T)^{\frac{2}{3}} \left( 10 \alpha \log T \right)^{\frac{1}{3}}\label{:pr:993},
\end{align}
where line~\eqref{:pr:991} is obtained by splitting the summation based on the value of $j^*$, line~\eqref{:pr:992} comes from bounding the summation with the integral (Lemma~\ref{lemma:tecSum}), and line~\eqref{:pr:993} follows from substituting the value of $j^*$ and simple algebraic manipulations.
Putting all together, we obtain:
\begin{align*}
R_{\bm{\mu}}(\text{\algrested},T) & \le 1+ \frac{2K}{\alpha-2} + 5K+ \frac{K}{\epsilon}+  \frac{3K}{\epsilon}(2\sigma T)^{\frac{2}{3}} \left( 10 \alpha \log T \right)^{\frac{1}{3}} + K T^q \left\lceil \frac{1}{1-2\epsilon}  \right\rceil  \Upsilon_{\bm{\mu}}\left( \left\lceil (1-2\epsilon) \frac{T}{K} \right\rceil,q\right)\\
& = \BigO \left( K T^q \left\lceil \frac{1}{1-2\epsilon}  \right\rceil  \Upsilon_{\bm{\mu}}\left( \left\lceil (1-2\epsilon) \frac{T}{K} \right\rceil,q\right) +  \frac{K}{\epsilon}(\sigma T)^{\frac{2}{3}} \left(\alpha \log T \right)^{\frac{1}{3}} \right) .
%
\end{align*}
\end{proof}

\subsection{Proofs of Section~\ref{sec:rrestless}}\label{apx:prrestless}

\thrRestlessOptimal*
\begin{proof}
	Trivially follows from the fact that the greedy policy at each round $t$ is selecting the largest expected reward, therefore any optimal policy other than the greedy one should select a larger expected reward at least for a single round $t'$, which is in contradiction with the definition of greedy policy.
\end{proof}

\begin{restatable}[]{lemma}{lemmaDetRestless}\label{lemma:lemmaDetRestless}
	For every arm  $i\in [K]$ and every round $t \in [T]$, let us define:
	\begin{align*}
		\overline{\mu}_i^{\text{\rless}}& (t)  \coloneqq \mu_i(t_{i,N_{i,t-1}}) + (t - t_{i,N_{i,t-1}}) \frac{\mu_i(t_{i,N_{i,t-1}}) - \mu_i(t_{i,N_{i,t-1}-1})}{t_{i,N_{i,t-1}} - t_{i,N_{i,t-1}-1}},
	\end{align*}
	if $N_{i,t-1} \ge 2$ else $\overline{\mu}_i^{\text{\rless}} (t) \coloneqq +\infty$. Then, $\overline{\mu}_i^{\text{\rless}}(t) \ge  \mu_i(t)$ and, if $N_{i,t-1} \ge 2$, it holds that:
	\begin{align*}
		\overline{\mu}_i^{\text{\rless}}(t) - \mu_i(t) \le \left(t - t_{i,N_{i,t-1}}\right) \gamma_i(t_{i,N_{i,t-1}-1}) .
	\end{align*}
\end{restatable}
\begin{proof}
	Let us consider the following derivation:
	\begin{align}
		\mu_i(t) & = \mu_i(t_{i,N_{i,t-1}}) + \sum_{l=t_{i,N_{i,t-1}}}^{t-1} \gamma_i(l) \notag\\
		& \le \mu_i(t_{i,N_{i,t-1}}) + (t - t_{i,N_{i,t-1}}) \gamma_i(t_{i,N_{i,t-1}}) \label{rls:001}\\
		&  \le \mu_i(t_{i,N_{i,t-1}}) + (t - t_{i,N_{i,t-1}}) \frac{\mu_i(t_{i,N_{i,t-1}}) - \mu_i(t_{i,N_{i,t-1}-1})}{t_{i,N_{i,t-1}}-t_{i,N_{i,t-1}-1}} \eqqcolon \overline{\mu}_i^{\text{\rless}}(t),\label{rls:002}
	\end{align}
	where line~\eqref{rls:001} follows from Assumption~\ref{ass:decrDeriv} and line~\eqref{rls:002} from Lemma~\ref{lemma:boundDeriv}. 
	Moreover, if $N_{i,t-1} \ge 2$, we have:
	\begin{align*}
		 \overline{\mu}_i^{\text{\rless}}(t) - \mu_i(t) & = (t - t_{i,N_{i,t-1}}) \frac{\mu_i(t_{i,N_{i,t-1}}) - \mu_i(t_{i,N_{i,t-1}-1})}{t_{i,N_{i,t-1}}-t_{i,N_{i,t-1}-1}} + \underbrace{\mu_i(t_{i,N_{i,t-1}}) -  \mu_i(t)}_{\leq 0}\\
		 & \le (t - t_{i,N_{i,t-1}}) \frac{\mu_i(t_{i,N_{i,t-1}}) - \mu_i(t_{i,N_{i,t-1}-1})}{t_{i,N_{i,t-1}}-t_{i,N_{i,t-1}-1}} \\
		 & = \frac{t - t_{i,N_{i,t-1}}}{t_{i,N_{i,t-1}}-t_{i,N_{i,t-1}-1}} \sum_{l=t_{i,N_{i,t-1}-1}}^{t_{i,N_{i,t-1}}-1} \gamma_i(l), \\
		 & \le \left(t - t_{i,N_{i,t-1}}\right) \gamma_i(t_{i,N_{i,t-1}-1}),
	\end{align*}
	where we employed Assumption~\ref{ass:decrDeriv} in the last line, noting $\frac{1}{t_{i,N_{i,t-1}}-t_{i,N_{i,t-1}-1}} \sum_{l=t_{i,N_{i,t-1}-1}}^{t_{i,N_{i,t-1}}-1} \gamma_i(l) \le \gamma_i(t_{i,N_{i,t-1}-1})$.
\end{proof}

\theRegretDetRestless*
\begin{proof}
We have to analyze the following expression:
\begin{align*}
	R_{\bm{\mu}}(\text{\algrestless},T) = \sum_{t=1}^T \mu_{i^*_t}(t) - \mu_{I_t}(t) ,
\end{align*}
where $i^*_t \in \argmax_{i \in [K]}\left\{\mu_i(t) \right\}$ for all $t \in [T]$. We consider each round at a time, recalling that $B_i(t) \equiv \overline{\mu}_i^{\text{\rless}}(t)$, and using optimism, \ie $B_{i^*_t}(t) \le B_{I_t}(t)$, we have:
\begin{align}
	\mu_{i^*}(t) - \mu_{I_t}(t) + B_{I_t}(t) - B_{I_t}(t) & \le  \min \left\{ 1, \underbrace{\mu_{i^*_t}(t) - B_{i^*_t}(t)}_{\le 0} + B_{I_t}(t) - \mu_{I_t}(t) \right\} \notag \\
	& \le \min \left\{ 1, B_{I_t}(t) - \mu_{I_t}(t)\right\}. \label{eq:boundK}
\end{align}
Now we consider the term inside the minimum, when $N_{I_t,t-1} \ge 2$:
\begin{align}
 B_{I_t}(t) - \mu_{I_t}(t) & = \overline{\mu}_{I_t}^{\text{\rless}}(t) - \mu_{I_t}(t) \\
 & \le  (t - t_{i,N_{i,t-1}}) \gamma_i(t_{i,N_{i,t-1}-1}), \label{eq:quart}
\end{align}
where to get line~\eqref{eq:quart} we applied Lemma~\ref{lemma:lemmaDetRestless}. Let us plug the expression derived in Equation~\eqref{eq:boundK} and decompose the summation of this term w.r.t.~the $K$ arms:
\begin{align}
	R_{\bm{\mu}}(\text{\algrestless},T) & \le \sum_{t=1}^T  \min \left\{ 1, (t - t_{i,N_{i,t-1}}) \gamma_i(t_{i,N_{i,t-1}-1}),\right\}\notag \\
& = 2K + \sum_{i \in [K]} \sum_{j=3}^{N_{i,T}} \min \left\{ 1, (t_{i,j} - t_{i,j-1}) \gamma_i(t_{i,j-2}) \right\}\notag\\
& \le 2K + \sum_{i \in [K]} \sum_{j=3}^{N_{i,T}} (t_{i,j} - t_{i,j-1})^y \gamma_i(t_{i,j-2})^y \label{:prl:001} \\
& \le 2K + \sum_{i \in [K]} \left(\sum_{j=3}^{N_{i,T}} (t_{i,j} - t_{i,j-1}) \right)^y \left(\sum_{j=3}^{N_{i,T}} \gamma_i(t_{i,j-2})^{\frac{y}{1-y}}\right)^{1-y} \label{:prl:002} \\
& \le 2K + T^y \sum_{i \in [K]} \left(\sum_{j=3}^{N_{i,T}} \gamma_i(j-2)^{\frac{y}{1-y}}\right)^{1-y} \label{:prl:003} \\
& \le 2K + T^y K^y\left( \sum_{i \in [K]} \sum_{j=3}^{N_{i,T}} \gamma_i(j-2)^{\frac{y}{1-y}} \right)^{1-y} \label{:prl:004}  \\
& \le 2K + T^y K \Upsilon_{\bm{\mu}}\left(\left\lceil \frac{T}{K} \right\rceil,\frac{y}{1-y} \right)^{1-y} \label{:prl:005},
\end{align}
 line~\eqref{:prl:001} follows from  the inequality $\min\{1,x\} \le \min\{1,x\}^{y} \le x^{y}$ for $y \in \left[0, \frac{1}{2}\right]$, line~\eqref{:prl:002} follows from H\"older's inequality with powers $\frac{1}{y} \geq 1$ and $\frac{1}{1-y} \geq 1$ (since $y \in \left[0, \frac{1}{2} \right]$), line~\eqref{:prl:003} is obtained from observing that $\sum_{j=3}^{N_{i,T}} (t_{i,j} - t_{i,j-1}) \le T$ and $\gamma_i(t_{i,j-2}) \le \gamma_i(j-2)$ from Assumption~\ref{ass:decrDeriv}, line~\eqref{:prl:004} follows from Jensen's inequality as $y \in \left[0, \frac{1}{2} \right]$ and observing:
\begin{align*}
\sum_{i \in [K]} \left(\sum_{j=3}^{N_{i,T}} \gamma_i(j-2)^{\frac{y}{1-y}}\right)^{1-y} = K \sum_{i \in [K]} \frac{1}{K} \left(\sum_{j=3}^{N_{i,T}} \gamma_i(j-2)^{\frac{y}{1-y}}\right)^{1-y} \le K^y \left( \sum_{i \in [K]} \sum_{j=3}^{N_{i,T}} \gamma_i(j-2)^{\frac{y}{1-y}} \right)^{1-y},
\end{align*}
where line~\eqref{:prl:005} is obtained from Lemma~\ref{lemma:boundUpsilonMax}.
The final theorem statement is obtained by defining $q := \frac{y}{1-y} \in [0,1]$ and substituting it to the above equation.

\end{proof}

\paragraph{Estimator Construction for the Stochastic Rising Restless  Setting}
We provide the intuition behind the estimator construction and explain why it differs significantly from the one employed for the deterministic case. We start observing that for every $l\in \{2, \dots, N_{i,t-1}\}$, we have that:
\begin{align*}
	\mu_i(t) = \textcolor{color11}{\underbrace{\mu_i(t_{i,l})}_{\text{(past payoff)}}} + \textcolor{color12}{\underbrace{\sum_{j=t_l}^{t-1} \gamma_i(j)}_{\text{(sum of future increments)}}} \le  \textcolor{color11}{\underbrace{\mu_i(l)}_{\text{(past payoff)}}} + \textcolor{color12}{(t-t_{i,l})\underbrace{\gamma_i(t_{i,l-1})}_{\text{(past increment)}}},
\end{align*}
where the inequality follows from Assumption~\ref{ass:decrDeriv}. Since do not have access to $\gamma_i(t_{i,l-1})$ and we cannot directly estimate it, we need to perform a further bounding step. Specifically, based on Lemma~\ref{lemma:boundDeriv}, we bound for every $l\in \{2, \dots, N_{i,t-1}\}$ and $h \in[ l -1]$:
\begin{align*}
\textcolor{color12}{\underbrace{\gamma_i(t_{i,l-1})}_{\text{(past increment at $t_{i,l}$)}}} \le \textcolor{color12}{\underbrace{\frac{\mu_i(t_{i,l}) - \mu_i(t_{i,l-h})}{t_{i,l} - t_{i,l-h}}}_{\text{(average past increment over $\{t_{i,l-h},\dots,t_{i,l}\}$)}}}.
\end{align*}
We report a first proposal of optimistic approximation of $\mu_i(t)$, \ie $\widetilde{\widetilde{\mu}}_i^{\text{\red},h}(t)$, and the corresponding estimator, \ie $\widehat{\widehat{\mu}}_i^{\text{\red},h}(t)$, that are defined in terms of a window of size $1 \le h \le \lfloor N_{i,t-1}/2 \rfloor$:
\begin{align*}
&  \widetilde{\widetilde{\mu}}_i^{\text{\rless},h} (t)  \coloneqq \frac{1}{h} \sum_{l=N_{i,t-1}-h+1}^{N_{i,t-1}} \Bigg(\textcolor{color11}{\underbrace{\mu_i(t_{i,l})}_{\text{(past payoff)}}} + \textcolor{color12}{(t-t_{i,l}) \underbrace{\frac{\mu_i(t_{i,l}) - \mu_{i}(t_{i,l-h})}{t_{i,l} - t_{i,l-h}}}_{\text{(average past increment)}}}\Bigg),\\
& \widehat{\widehat{\mu}}_i^{\text{\rless},h} (t)  \coloneqq \frac{1}{h} \sum_{l=N_{i,t-1}-h+1}^{N_{i,t-1}} \Bigg(\textcolor{color11}{\underbrace{R_{t_{i,l}}}_{\text{(estimated past payoff)}}} + \textcolor{color12}{(t-t_{i,l}) \underbrace{\frac{R_{t_{i,l}} - R_{t_{i,l-h}}}{t_{i,l} - t_{i,l-h}}}_{\text{(estimated average past increment)}}}\Bigg).
\end{align*}
Unfortunately, this estimator, although intuitive, does not enjoy desirable concentration properties due to the presence of the denominator $t_{i,l} - t_{i,l-h}$ that is inconveniently correlated with the numerator $R_{t_{i,l}} - R_{t_{i,l-h}}$. For this reason, we resort to different estimators, with better concentration properties but larger bias:
\begin{align*}
&  {\widetilde{\mu}}_i^{\text{\rless},h} (t)  \coloneqq \frac{1}{h} \sum_{l=N_{i,t-1}-h+1}^{N_{i,t-1}} \Bigg(\textcolor{color11}{\underbrace{\mu_i(t_{i,l})}_{\text{(past payoff)}}} + \textcolor{color12}{(t-l) \underbrace{\frac{\mu_i(t_{i,l}) - \mu_{i}(t_{i,l-h})}{h}}_{\text{(average past increment)}}}\Bigg),\\
& {\widehat{\mu}}_i^{\text{\rless},h} (t)  \coloneqq \frac{1}{h} \sum_{l=N_{i,t-1}-h+1}^{N_{i,t-1}} \Bigg(\textcolor{color11}{\underbrace{R_{t_{i,l}}}_{\text{(estimated past payoff)}}} + \textcolor{color12}{(t-l) \underbrace{\frac{R_{t_{i,l}} - R_{t_{i,l-h}}}{h}}_{\text{(estimated average past increment)}}}\Bigg).
\end{align*}
These estimators are actually upper-bounds of the previous ones since $t_{i,l} - t_{i,l-h} \ge h$ and $t_{i,l} \ge l$.

The proof is composed of the following steps:
\begin{enumerate}[label=(\roman*)]
	\item Lemma~\ref{lemma:firstEstimatorRestless} shows that $\widetilde{\mu}_i^{\text{\rless},h}(t)$ is an upper-bound for $\mu_i(t)$ and provides a bound to its bias for every value of $h$;
	\item Lemma~\ref{lemma:firstEstimatorConcRestless} analyzes the concentration of $\widehat{\mu}_i^{\text{\rless},h}(t)$ around $\widetilde{\mu}_i^{\text{\rless},h}(t)$ for a specific choice of $\delta_t = t^{-\alpha}$ and when $h_{i,h} \coloneqq h(N_{i,t-1})$ is a function of the number of pulls $N_{i,t-1}$ only;
	\item Theorem~\ref{thr:theRegretRestlessStochastic} bounds the expected regret of \algrestless when $h_{i,h} =  \lfloor \epsilon N_{i,t-1} \rfloor$ for $\epsilon \in (0,1/2)$.
\end{enumerate}

\begin{restatable}[]{lemma}{firstEstimatorRestless}\label{lemma:firstEstimatorRestless}
For every arm $i \in [K]$, every round $t \in [T]$, and window width $1 \le h \le \lfloor N_{i,t-1}/2 \rfloor$, let us define:
\begin{align*}
	 \widetilde{\mu}_i^{\text{\rless},h}(t) &\coloneq \frac{1}{h} \sum_{l=N_{i,t-1}-h+1}^{N_{i,t-1}} \bigg(\mu_i(t_{i,l}) + ( t-l)\frac{\mu_i(t_{i,l}) - \mu_{i}(t_{i,l-h})}{h} \bigg) ,
\end{align*}
otherwise if $h=0$, we set $\widetilde{\mu}_i^{\text{\rless},h}(t)  \coloneqq +\infty$. Then, $\widetilde{\mu}_i^{\text{\rless},h}(t) \ge \mu_i(t_{i,N_{i,t-1}})$ and, if $N_{i,t-1} \ge 2$ it holds that:
\begin{align*}
	 \widetilde{\mu}_i^{\text{\rless},h }  (t)   - \mu_i(t) \le \frac{(2t-2N_{i,t-1}+h-1)(t_{i,N_{i,t-1}} - t_{i,N_{i,t-1}-2h+1})}{2h} \gamma_{i}(t_{i,N_{i,t-1}-2h+1}).
\end{align*}
\end{restatable}
\begin{proof}
	Let us start by observing the following equality holding for every $l \in \{2,\dots,N_{i,t-1}\}$:
	\begin{align*}
		\mu_i(t) = \mu_i(t_{i,l}) + \sum_{j=t_{i,l}}^{t-1} \gamma_i(j).
	\end{align*}
	By averaging over a window of length $h$, we obtain:
	\begin{align}
		\mu_i(t) & = \frac{1}{h} \sum_{l=N_{i,t-1}-h+1}^{N_{i,t-1}} \left(\mu_i(t_{i,l}) +  \sum_{j=t_{i,l}}^{t-1} \gamma_i(j) \right) \notag\\
		& \le \frac{1}{h} \sum_{l=N_{i,t-1}-h+1}^{N_{i,t-1}} \left(\mu_i(t_{i,l}) +  (t - t_{i,l} )\gamma_i(t_{i,l}-1) \right) \label{:prs:3001}\\
		& \le \frac{1}{h} \sum_{l=N_{i,t-1}-h+1}^{N_{i,t-1}} \left(\mu_i(t_{i,l}) +  \frac{t - t_{i,l} }{t_{i,l} - t_{i,l-h}} \sum_{j={t_{i,l-h}}}^{t_{i,l}-1} \gamma_i(j) \right)  \label{:prs:3002}\\
		& = \frac{1}{h} \sum_{l=N_{i,t-1}-h+1}^{N_{i,t-1}} \left(\mu_i(t_{i,l}) +  (t - t_{i,l})\frac{\mu_i(t_{i,l}) - \mu_i(t_{i,l-h}) }{t_{i,l} - t_{i,l-h}}  \right)  \notag \\
		& \le \frac{1}{h} \sum_{l=N_{i,t-1}-h+1}^{N_{i,t-1}} \left(\mu_i(t_{i,l}) +  (t - l)\frac{\mu_i(t_{i,l}) - \mu_i(t_{i,l-h}) }{h}  \right) \eqqcolon \widetilde{\mu}_i^{\text{\rless},h}(t),  \label{:prs:3003}
	\end{align}
	where lines~\eqref{:prs:3001} and \eqref{:prs:3002} follow from Assumption~\ref{ass:decrDeriv}, and line~\eqref{:prs:3003} is obtained from observing that $t_{i,l} \ge l$ and $t_{i,l}- t_{i,l-h} \ge h$.
	 Concerning the bias, when $N_{i,t-1} \ge 2$, we have:
	 \begin{align}
	 \widetilde{\mu}_i^{\text{\rless},h}(t)- \mu_i(t) & = \frac{1}{h} \sum_{l=N_{i,t-1}-h+1}^{N_{i,t-1}} \left( \mu_i(t_{i,l}) +  (t-l)\frac{\mu_i(t_{i,l}) - \mu_{i}(t_{i,l-h})}{h}  \right) - \mu_i(t)\notag \\
	 & \le \frac{1}{h} \sum_{l=N_{i,t-1}-h+1}^{N_{i,t-1}}   (t-l)\frac{\mu_i(t_{i,l}) - \mu_{i}(t_{i,l-h})}{h}   \label{:prs:4001}\\
	  &= \frac{1}{h} \sum_{l=N_{i,t-1}-h+1}^{N_{i,t-1}}   (t-l)\frac{\mu_i(t_{i,l}) - \mu_{i}(t_{i,l-h})}{t_{i,l} - t_{i,l-h}} \cdot \frac{t_{i,l} - t_{i,l-h}}{h}   \notag\\
	  &\le \frac{1}{h} \sum_{l=N_{i,t-1}-h+1}^{N_{i,t-1}}   (t-l) \gamma_{i}(t_{i,l-h}) \cdot \frac{t_{i,l} - t_{i,l-h}}{h}   \label{:prs:4004}\\
	  &\le \frac{t_{i,N_{i,t-1}} - t_{i,N_{i,t-1}-2h+1}}{h^2} \gamma_{i}(t_{i,N_{i,t-1}-2h+1}) \sum_{l=N_{i,t-1}-h+1}^{N_{i,t-1}}   (t-l)   \label{:prs:4005}\\
	  & = \frac{(2t-2N_{i,t-1}+h-1)(t_{i,N_{i,t-1}} - t_{i,N_{i,t-1}-2h+1})}{2h} \gamma_{i}(t_{i,N_{i,t-1}-2h+1}), \label{:prs:4006}
	 \end{align}
	 where line~\eqref{:prs:4001} follows from observing that $\mu_i(t_{i,l}) \le \mu_i(t)$, line~\eqref{:prs:4004} derives from Assumption~\ref{ass:decrDeriv} and bounding $\frac{\mu_i(t_{i,l}) - \mu_{i}(t_{i,l-h})}{t_{i,l} - t_{i,l-h}} \le \gamma_{i}(t_{i,l-h})$, line~\eqref{:prs:4005} is obtained by bounding $t_{i,l} - t_{i,l-h} \le t_{i,N_{i,t-1}} - t_{i,N_{i,t-1}-2h+1}$ and $\gamma_{i}(t_{i,l-h}) \le \gamma_{i}(t_{i,N_{i,t-1}-2h+1})$, and line~\eqref{:prs:4006} follows from computing the summation.
\end{proof}

\begin{restatable}[]{lemma}{firstEstimatorConc}\label{lemma:firstEstimatorConcRestless}
For every arm $i \in [K]$, every round $t \in [T]$, and window width $1 \le h \le \lfloor N_{i,t-1} / 2\rfloor$, let us define:
\begin{align*}
	 & \widehat{\mu}_i^{\text{\rless},h} (t)  \coloneqq \frac{1}{h} \sum_{l=N_{i,t-1}-h+1}^{N_{i,t-1}} \left(R_{t_{i,l}} + (t-l) \frac{R_{t_{i,l}} - R_{t_{i,l-h}}}{h}\right),\\
	 & \beta^{\text{\rless},h}_i(t,\delta)\coloneqq \sigma  (t-N_{i,t-1}+h-1) \sqrt{ \frac{ 10  \log \frac{1}{\delta} }{h^3} },
\end{align*}
otherwise if $h=0$, we set $\widehat{\mu}_i^{\text{\rless},h} (t) \coloneqq +\infty$ and $\beta^{\text{\rless},h}_i(t,\delta)\coloneqq +\infty$.
Then, if the window size depends on the number of pulls only $h_{i,t} = h(N_{i,t-1}) $ and if $\delta_t = t^{-\alpha}$ for some $\alpha > 2$, it holds for every round  $t \in [T]$ that:
\begin{align*}
	\Pr \left(\left| \widehat{\mu}_i^{\text{\rless},h_{i,t}}(t) - \widetilde{\mu}_i^{\text{\rless},h_{i,t}}(t) \right| >  \beta^{\text{\rless},h_{i,t}}_i(t,\delta_t) \right) \le 2t^{1-\alpha}.
\end{align*}
\end{restatable}

\begin{proof}
	Under the event $\{ h_{i,t} = 0\}$, we have that $\widehat{\mu}_i^{\text{\rless},h_{i,t}} (t)= \widetilde{\mu}_i^{\text{\rless},h_{i,t}} (t) = \beta^{\text{\rless},h_{i,t}}_i(t,\delta) = +\infty$ and, under the convention $(+\infty) - (+\infty) = 0$ the event $0 > \beta^{\text{\rless},h}_i(t,\delta)$ does not hold. Therefore, we conduct the proof under the event $\{ h_{i,t} \ge 1\}$. Hence:
	\begin{align}
	\Pr & \left(\left| \widehat{\mu}_i^{\text{\rless},h_{i,t}}(t) - \widetilde{\mu}_i^{\text{\rless},h_{i,t}}(t) \right| >  \beta^{\text{\rless},h_{i,t}}_i(t,\delta_t) \right) \\
	& \qquad\qquad =  \Pr \left(\left| \widehat{\mu}_i^{\text{\rless},h(N_{i,t-1})}(t) - \widetilde{\mu}_i^{\text{\rless},h(N_{i,t-1})}(t) \right| >  \beta^{\text{\rless},h(N_{i,t-1})}_i(t,\delta_t) \right) \label{pah:01} \\
	& \qquad\qquad\le   \Pr \left( \exists n \in \{0,\dots, t-1\} \text{ s.t. } h(n) \ge 1\,:\, \left| \widehat{\mu}_i^{\text{\rless},h(n)}(t) - \widetilde{\mu}_i^{\text{\rless},h(n)}(t) \right| >  \beta^{\text{\rless},h(n)}_i(t,\delta_t) \right) \notag \\
	& \qquad\qquad\le \sum_{n\in\{0,\dots, t-1\} \,:\, h(n) \ge 1} \Pr \left(\left| \widehat{\mu}_i^{\text{\rless},h(n)}(t) - \widetilde{\mu}_i^{\text{\rless},h(n)}(t) \right| >  \beta^{\text{\rless},h(n)}_i(t,\delta_t) \right),\label{pah:02}
	\end{align}
	where line~\eqref{pah:01} follows from the definition of $h_{i,t} = h(N_{i,t-1})$, and line~\eqref{pah:02} derives from a union bound over $n$. Differently from the rested case, in which the distribution of all random variable involved is fully determined having fixed $N_{i,t-1}$, in the restless case this is no longer the case. Indeed, the distribution of the rewards does not depend on the number of pulls, but on the round in which the arm was pulled. Thus, we need a more articulated argument. We start rewriting the estimator with a summation over rounds:
	\begin{align}\label{eq:eqeqeq}
	 h(n) \left( \widehat{\mu}_i^{\text{\rless},h(n)}(t) - \widetilde{\mu}_i^{\text{\rless},h(n)}(t) \right) & =  \sum_{l=n-h(n)+1}^{n} \left(X_{t_{i,l}} + (t-l) \frac{X_{t_{i,l}} - X_{t_{i,l-h(n)}}}{h(n)}\right) =  \sum_{s=1}^{t-1} \epsilon_s Y_s X_s ,
	 \end{align}
	 where:
	 \begin{align*}
	 & \epsilon_{s} = \mathds{1}\{I_s = i\}, \\
	 & Y_s =  \left( \mathds{1} \left\{ N_{i,s} \in \left\{ n - h(n)+1 , \dots, n \right\} \right\} \left(1 + \frac{t - N_{i,s}}{h(n)} \right) - \mathds{1} \left\{ N_{i,s} \in \left\{ n - 2h(n)+1 , \dots, n-h(n) \right\} \right\} \frac{t - N_{i,s} - h(n)}{h(n)}  \right), \\
	 & X_{s} = R_s - \mu_{i}(s).
	\end{align*}	 
	The rationale behind this decomposition is to use random variable $\epsilon_s$ to select the pulls of arm $i$, $Y_s$ to define the quantity by which $X_s$  is multiplied. In particular, if the pull belongs to the set of the most recent $h(n)$ pulls, \ie $N_{i,s} \in \left\{ n - h(n)+1 , \dots, n \right\}$, we multiply $X_s$ by the constant $1 + \frac{t - N_{i,s}}{h(n)}$. Instead, if the pull belongs to less recent $h(n)$ pulls, \ie $N_{i,s} \in \left\{ n - 2h(n)+1 , \dots, n-h(n) \right\} $, we multiply $X_s$ by $\frac{t - N_{i,s} - h(n)}{h(n)} $. Now, we define the sequence of random times at which arm $i$ was pulled for the $j$-th time:
	\begin{align*}
		t_{i,j} \coloneqq \min_{t \in [T]} \left\{ N_{i,t} = j \right\}, \quad j \in [n],
	\end{align*}
	and we introduce the random variables $\widetilde{X}_j \coloneqq X_{t_{i,j}}$ and $\widetilde{Y}_{j} \coloneqq Y_{t_{i,j}}$. To prove that $ \widetilde{Y}_{j} \widetilde{X}_j $ is a martingale difference sequence \wrt to the filtration it generates, we apply a Doob's \emph{optional skipping} argument~\cite{doob1953stochastic, BubeckMSS08}. We introduce the filtration $\mathcal{F}_{\tau-1} = \sigma \left(I_1,R_1, \dots, I_{\tau-1},R_{\tau-1},I_\tau \right)$ and we need to show that: (i) $Z_\tau = \sum_{s=1}^{\tau} \epsilon_s Y_s X_s $ is a martingale, and (ii) $\{t_{i,j} = \tau\} \in \mathcal{F}_{\tau-1}$ for $\tau \in [t-1]$. Concerning (i), we have:
		\begin{align*}
		\E\left[ Z_\tau | \mathcal{F}_{\tau-1} \right] = Z_{\tau-1} + \epsilon_\tau Y_\tau  \E\left[ X_\tau | \mathcal{F}_{\tau-1} \right] = Z_{\tau-1},
		\end{align*}
		since $\epsilon_\tau Y_\tau$ is fully determined by $\mathcal{F}_{\tau-1}$ and either $\epsilon_\tau = 0$ or $I_\tau = i$, thus, $\epsilon_\tau\E\left[ X_\tau | \mathcal{F}_{\tau-1} \right] = \epsilon_\tau \E\left[ R_\tau - \mu_{i}(\tau) | \mathcal{F}_{\tau-1}\right] = 0$. Concerning (ii), $\{t_{i,j} = \tau\} \in \mathcal{F}_{\tau-1}$ is trivially verified. We recall that, since $\widetilde{Y}_{j} = Y_{t_{i,j}}$ we have that $N_{i,t_{i,j}} = j$:
		\begin{align*}
			\widetilde{Y}_{j} = \left( \mathds{1} \left\{ j \in \left\{ n - h(n)+1 , \dots, n \right\} \right\} \left(1 + \frac{t - j}{h(n)} \right) - \mathds{1} \left\{ j \in \left\{ n - 2h(n)+1 , \dots, n-h(n) \right\} \right\} \frac{t - j - h(n)}{h(n)}  \right).
		\end{align*}
		From which, by substituting into Equation~\eqref{eq:eqeqeq} and properly solving the indicator functions, we have:
		\begin{align*}
		\sum_{j=1}^n \widetilde{X}_j \widetilde{Y}_{j} = \sum_{j=n-h(n)+1}^n  \left( 1+ \frac{t-j}{h(n)} \right)\widetilde{X}_j - \sum_{j=n-2h(n)+1}^{n-h(n)}  \frac{t-j}{h(n)} \cdot \widetilde{X}_j . 
		\end{align*}
		We compute the square of the weights and apply a derivation similar to that of Lemma~\ref{lemma:firstEstimatorConc}:
		\begin{align*}
			\sum_{j=n-h(n)+1}^n \left( 1+ \frac{t-j}{h(n)} \right)^2 + \sum_{j=n-2h(n)+1}^{n-h(n)} \left(\frac{t-j}{h(n)}\right)^2 \le \frac{5(t-n+h(n)-1)^2}{h(n)}.
		\end{align*}
		Thus, we can now apply Azuma-H\"oeffding's inequality (Lemma~\ref{lemma:hoeffding}):
		\begin{align*}
		\Pr & \left(\left| \widehat{\mu}_i^{\text{\rless},h(n)}(t) - \widetilde{\mu}_i^{\text{\rless},h(n)}(t) \right| >  \beta^{\text{\rless},h(n)}_i(t,\delta_t) \right) \\
		& \qquad\qquad= \Pr \left( \left| \sum_{s=1}^t \epsilon_s X_s Y_s \right| > h(n) \beta^{\text{\rless},h(n)}_i(t,\delta_t)  \right) \\
		& \qquad\qquad=  \Pr \left( \left| \sum_{j=1}^n \widetilde{X}_j \widetilde{Y}_{j} \right| > h(n) \beta^{\text{\rless},h(n)}_i(t,\delta_t)  \right) \\
		&  \qquad\qquad \le 2 \exp \left( - \frac{\left(h(n)\beta^{\text{\red},h(n)}_i(t,\delta_t) \right)^2}{2 \sigma^2 \left( \frac{ 5(t-n+h(n)-1)^2}{h(n)} \right)}\right) = 2 \delta_t.
		\end{align*}
		By replacing into Equation~\eqref{pah:02} and summing over $n$, we obtain:
		\begin{align*}
			\sum_{n\in\{0,\dots, t-1\} \,:\, h(n) \ge 1} 2 \delta_t \le \sum_{n=0}^{t-1} 2 \delta_t = 2t^{1-\alpha}.
		\end{align*}
\end{proof}

\theRegretRestlessStochastic*
\begin{proof}
Let us define the good events $\mathcal{E}_t = \bigcap_{i \in [K]} \mathcal{E}_{i,t}$ that correspond to the event in which all confidence intervals hold:
\begin{align*} 
\mathcal{E}_{i,t} \coloneqq \left\{ \left| \widetilde{\mu}_i^{\text{\rless},h_{i,t}}(t) - \widehat{\mu}_i^{\text{\rless},h_{i,t}}(t)\right| \le \beta^{\text{\rless},h_{i,t}}_i(t)\right\} \qquad \forall i \in [T], \,i \in [K].
\end{align*}

We have to analyze the following expression:
\begin{align*}
	R_{\bm{\mu}}(\text{\algrestless},T) =  \mathbb{E} \left[\sum_{t=1}^T  \mu_{i^*_t}(t) - \mu_{I_t}(t) \right] ,
\end{align*}
where $i^*_t \in \argmax_{i \in [K]}\left\{\mu_i(t) \right\}$ for all $t \in [T]$. 
 We decompose according to the good events $\mathcal{E}_t$:
\begin{align*}
R_{\bm{\mu}}(\text{\algrestless},T) & = \sum_{t=1}^T \mathbb{E} \left[\left( \mu_{i^*_t}(t) - \mu_{I_t}(t) \right)\mathds{1}\{\mathcal{E}_t\} \right] + \sum_{t=1}^T \mathbb{E} \left[\left( \mu_{i^*_t}(t) - \mu_{I_t}(t) \right)\mathds{1}\{\lnot\mathcal{E}_t\}\right] \\
& \le \sum_{t=1}^T \mathbb{E} \left[\left( \mu_{i^*_t}(t) - \mu_{I_t}(t) \right)\mathds{1}\{\mathcal{E}_t\}\right] + \sum_{t=1}^T \mathbb{E} \left[\mathds{1}\{\lnot\mathcal{E}_t\}\right], 
\end{align*}
where we exploited $\mu_{i^*_t}(t) - \mu_{I_t}(t) \le 1$ in the inequality.
Now, we bound the second summation, as done in Theorem~\ref{thr:theRegretRestedStochastic}:
\begin{align*}
\sum_{t=1}^T \mathbb{E} \left[\mathds{1}\{\lnot\mathcal{E}_t\}\right] \le 1+ \frac{2K}{\alpha-2}.
\end{align*}

From now on, we will proceed the analysis under the good event $\mathcal{E}_t$, recalling that $B_i(t) \equiv \widehat{\mu}_i^{\text{\rless},h_{i,t}}(t) + \beta^{\text{\rless},h_{i,t}}_i(t)$. Let $t \in [T]$, and we exploit the optimism, \ie $B_{i^*_t}(t) \le B_{I_t}(t)$:
\begin{align*}
	\mu_{i^*}(t) - \mu_{I_t}(t) + B_{I_t}(t) - B_{I_t}(t) & \le  \min \left\{ 1, \underbrace{\mu_{i^*_t}(t) - B_{i^*_t}(t)}_{\le 0} + B_{I_t}(t) - \mu_{I_t}(t) \right\} \\
	& \le \min \left\{ 1, B_{I_t}(t) - \mu_{I_t}(t)\right\}.
\end{align*}
Now, we work on the term inside the minimum:
\begin{align}
	B_{I_t}(t) - \mu_{I_t}(t) &= \widehat{\mu}_{I_t}^{\text{\rless},h_{I_t,t}}(t) + \beta^{\text{\rless},h_{I_t,t}}_{I_t}(t)  - \mu_{I_t}(t) \label{lprl:001}\\
	& \le \underbrace{\widetilde{\mu}_{I_t}^{\text{\rless},h_{I_t,t}}(t)  - \mu_{I_t}(t)}_{\text{(a)}} + \underbrace{2\beta^{\text{\rless}, h_{I_t,t}}_{I_t}(t)}_{\text{(b)}},\label{lprl:002}
\end{align}	
where line~\eqref{lprl:001} follows from the definition of $B_i(t)$ and line~\eqref{lprl:002} from the good event $\mathcal{E}_t$. We proceed decomposing over the arms, starting with (a):
\begin{align}
\sum_{t=1}^T & \min\left\{1, \widetilde{\mu}_{I_t}^{\text{\rless},h_{I_t,t}}(t) - \mu_{I_t}(t) \right\}  \le 2K + \sum_{i \in [K]} \sum_{j=3}^{N_{i,T}}\min \left\{1, \widetilde{\mu}_{i}^{\text{\rless},h_{i,t_{i,j}}}(t_{i,j}) - \mu_{i}(t_{i,j}) \right\} \notag\\
& \le 2K + \sum_{i \in [K]} \sum_{j=3}^{N_{i,T}}\min \left\{1, \frac{(2t_{i,j} - 2(j-1)+ h_{i,t_{i,j}} - 1)(t_{i,j-1} - t_{i,j-2h_{i,t}+1})}{2h_{i,t}} \gamma_{i}(t_{i,(j-1)-2h_{i,t_{i,j}}+1}) \right\}  \label{:prr:001}\\
& \le 2K + \sum_{i \in [K]} \sum_{j=3}^{N_{i,T}}\min \left\{1, \frac{T^2}{\lfloor \epsilon(j-1)\rfloor}  \gamma_{i}(t_{i,j-2\lfloor \epsilon(j-1)\rfloor}) \right\}  \label{:prr:003}\\
& \le 2K + \sum_{i \in [K]} \sum_{j=3}^{N_{i,T}}\min \left\{1, \frac{T^2}{\lfloor \epsilon(j-1)\rfloor}  \gamma_{i}(\lfloor (1-2\epsilon)j\rfloor) \right\}  \label{:prr:004}\\
&\le 2K + T^{2z} \sum_{i \in [K]} \sum_{j=3}^{N_{i,T}} \left( \frac{\gamma_{i}(\lfloor (1-2\epsilon)j\rfloor)}{\lfloor \epsilon(j-1)\rfloor} \right)^z   \label{:prr:005}\\
&\le 2K + T^{2z} \sum_{i \in [K]} \left(\sum_{j=3}^{N_{i,T}} \frac{1}{\lfloor \epsilon(j-1)\rfloor} \right)^{z} \left(\sum_{j=3}^{N_{i,T}} \gamma_{i}(\lfloor (1-2\epsilon)j\rfloor)^{\frac{z}{1-z}} \right)^{1-z}  \label{:prr:006}\\
&\le 2K + T^{2z} \left\lceil \frac{1}{\epsilon} \right\rceil \left\lceil \frac{1}{1-2\epsilon} \right\rceil \sum_{i \in [K]} \left(\sum_{j=\lfloor 2\epsilon \rfloor}^{\lfloor \epsilon (N_{i,T}-1)\rfloor } \frac{1}{j} \right)^{z}  \left(\sum_{j=\lfloor 3 (1-2\epsilon)\rfloor}^{ \lfloor(1-2\epsilon) N_{i,T} \rfloor } \gamma_{i}(j)^{\frac{z}{1-z}} \right)^{1-z}  \label{:prr:007}\\
&\le 2K + T^{2z} (1 + \log(\epsilon T))^z \left\lceil \frac{1}{\epsilon} \right\rceil \left\lceil \frac{1}{1-2\epsilon} \right\rceil \sum_{i \in [K]}  \left(\sum_{j=\lfloor 3 (1-2\epsilon)\rfloor}^{ \lfloor(1-2\epsilon) N_{i,T} \rfloor } \gamma_{i}(j)^{\frac{z}{1-z}} \right)^{1-z}  \label{:prr:008}\\
&\le 2K + T^{2z} (1 + \log(\epsilon T))^z \left\lceil \frac{1}{\epsilon} \right\rceil \left\lceil \frac{1}{1-2\epsilon} \right\rceil K^z \left( \sum_{i \in [K]}  \sum_{j=\lfloor 3 (1-2\epsilon)\rfloor}^{ \lfloor(1-2\epsilon) N_{i,T} \rfloor } \gamma_{i}(j)^{\frac{z}{1-z}} \right)^{1-z}  \label{:prr:009}\\
&\le 2K + T^{2z} (1 + \log(\epsilon T))^z \left\lceil \frac{1}{\epsilon} \right\rceil \left\lceil \frac{1}{1-2\epsilon} \right\rceil K \Upsilon_{\bm{\mu}}\left( \left\lceil(1-2\epsilon) \frac{T}{K} \right\rceil,  \frac{z}{1-z} \right) ^{1-z} . \label{:prr:0010}
\end{align}
where line~\eqref{:prr:001} follows from the bias bound of Lemma~\ref{lemma:firstEstimatorRestless}, line~\eqref{:prr:003} is obtained from bounding $(2t_{i,j} - 2(j-1)+ h_{i,t_{i,j}} - 1)(t_{i,j-1} - t_{i,j-2h_{i,t}+1}) \le 2 T^2$ and using the definition of $h_{i,t}$, line~\eqref{:prr:004} derives from observing that $\gamma_i(t_{i,j}) \le \gamma_i(j)$ for Assumption~\ref{ass:decrDeriv} and having bounded the floor analogously as done in Theorem~\ref{thr:theRegretRestedStochastic}, line~\eqref{:prr:005}  from the inequality $\min\{1,x\} \le \min\{1,x\}^z\le x^z$ for $z \in [0,1/2]$, line~\eqref{:prr:006} is obtained from H\"older's inequality with exponents $\frac{1}{z} \ge 1$ and $\frac{1}{1-z} \ge 1$ respectively, line~\eqref{:prr:007} is an application of Lemma~\ref{lemma:sumWithFloor} to independently to both inner summations, line~\eqref{:prr:008} derives from bounding the harmonic sum, \ie $\sum_{\lfloor 2 \epsilon  \rfloor}^{\lfloor \epsilon (N_{i,T}-1) \rfloor } \frac{1}{j} \le 1 + \log(\epsilon (N_{i,T}-1)) \le 1 + \log(\epsilon T)$, line~\eqref{:prr:009} follows from Jensen's inequality, line~\eqref{:prr:0010} is obtained from Lemma~\ref{lemma:boundUpsilonMax}. By recalling $q = \frac{z}{1-z} \in [0,1]$, we obtain:
\begin{align*}
2K + T^{\frac{2q}{1+q}} (1 + \log(\epsilon T))^{\frac{q}{1+q}} \left\lceil \frac{1}{\epsilon} \right\rceil \left\lceil \frac{1}{1-2\epsilon} \right\rceil K \Upsilon_{\bm{\mu}}\left( \left\lceil(1-2\epsilon) \frac{T}{K} \right\rceil,  q \right)^{\frac{1}{1+q}}.
\end{align*}

Concerning the term (b), we recall that $\beta^{\text{\rless}, h_{I_t,t}}_{I_t}(t)$ equals the bonus term used in the rested setting and, consequently from Theorem~\ref{thr:theRegretRestedStochastic}:
\begin{align*}
\sum_{t=1}^T \min\left\{1, 2 \beta^{\text{\red},h_{I_t,t}}_{I_t}(t,\delta_t)\right\} \le K \left(3 + \frac{1}{\epsilon} \right)+ \frac{3K}{\epsilon}(2\sigma T)^{\frac{2}{3}} \left( 10 \alpha \log T \right)^{\frac{1}{3}}.
\end{align*}

Putting all together, we obtain:
\begin{align*}
R_{\bm{\mu}}(\text{\algrestless},T) & \le 1+ \frac{2K}{\alpha-2} + 5K + \frac{K}{\epsilon}   + \frac{3K}{\epsilon}(2\sigma T)^{\frac{2}{3}} \left( 10 \alpha \log T \right)^{\frac{1}{3}}  \\
& \quad + T^{\frac{2q}{1+q}} (1 + \log(\epsilon T))^{\frac{q}{1+q}}\left\lceil \frac{1}{\epsilon} \right\rceil \left\lceil \frac{1}{1-2\epsilon} \right\rceil K \Upsilon_{\bm{\mu}}\left( \left\lceil(1-2\epsilon) \frac{T}{K} \right\rceil,  q \right)^{\frac{1}{1+q}}.
\end{align*}

\end{proof}
%

\newpage
\section{Bounding the Cumulative Increment}\label{apx:example}
Let us consider the case in which $\gamma_i(l) \le l^{-c}$ for all $i \in [K]$ and $l \in[T]$. We bound the cumulative increment with the corresponding integral using Lemma~\ref{lemma:tecSum}, depending on the value of $cq$:
\begin{align*}
	\Upsilon_{\bm{\mu}}\left(\left\lceil \frac{T}{K} \right\rceil , q\right) = \sum_{l=1}^{\left\lceil \frac{T}{K} \right\rceil } \max_{i \in [K]} \gamma_i(l)^q \le 1 + \int_{x=1}^{\frac{T}{K}} x^{-cq} \de x \le 1+ \begin{cases}
		\left(\frac{T}{K}\right)^{1-cq} \frac{1}{1-cq} & \text{if } cq < 1 \\
		\log \frac{T}{K} & \text{if } cq = 1\\
		\frac{1}{cq-1} & \text{if } cq > 1
	\end{cases}.
\end{align*}
Thus, depending on the value of $c$, there will be different optimal values for $q$ in the rested and restless cases that optimize the regret upper bound.

\subsection{Rested Setting}
Let us start with the rested case. From Theorem~\ref{thr:theRegretDetRested}, we have:
\begin{align*}
	R_{\bm{\mu}} & \le 2K + T^q K  \Upsilon_{\bm{\mu}}\left(\left\lceil \frac{T}{K} \right\rceil,q \right) \le 2K + K T^q + K \begin{cases}
		\frac{T^{1-cq+q}}{K^{1-qc}(1-cq)} & \text{if } cq < 1 \\
		T^q \log  \frac{T}{K} & \text{if } cq = 1\\
		\frac{T^q}{cq-1} & \text{if } cq > 1
	\end{cases}\\
	&  \le \BigO \left( K \begin{cases}
		\frac{T^{1-cq+q}}{K^{1-qc}(1-cq)} & \text{if } cq < 1 \\
		T^q \log \frac{T}{K} & \text{if } cq = 1\\
		\frac{T^q}{\min\{1,cq-1\}} & \text{if } cq > 1
	\end{cases} \right) \qquad \forall q \in [0,1],
\end{align*}
where we have highlighted the dominant term. For the case $c \in (0,1)$ we consider the first case only and minimize over $q$:
\begin{align*}
R_{\bm{\mu}} \le \BigO \left(K \min_{q \in [0,1]}\frac{T^{1-cq+q}}{K^{1-qc}(1-cq)} \right) =  \BigO \left(T \right).
\end{align*}
For the case $c=1$, we still obtain $R_{\bm{\mu}} \le\BigO \left(T \right)$. Instead, for $c \in (1,+\infty)$, we have the three cases:
\begin{align*}
	R_{\bm{\mu}} \le\BigO \left( K \min \begin{cases}
		K \min_{q \in [0,1/c)}\frac{T^{1-cq+q}}{K^{1-qc}(1-cq)}  \\
		T^\frac{1}{c} \log \frac{T}{K} \\
		\min_{q \in (1/c,1]} \frac{T^q}{\min\{1,cq-1\}} 
	\end{cases}\right) = \BigO\left( K T^\frac{1}{c} \log \frac{T}{K}\right).
\end{align*}

\subsection{Restless Setting}
Let us now move to the restless setting.  From Theorem~\ref{thr:theRegretDetRestless}, we have:
\begin{align*}
	R_{\bm{\mu}} & \le 2K + T^\frac{q}{q+1} K  \Upsilon_{\bm{\mu}}\left(\left\lceil \frac{T}{K} \right\rceil,q \right)^{\frac{1}{1+q}} \le 2K + K T^{\frac{q}{q+1}} + K \begin{cases}
		\frac{T^{\frac{1-cq+q}{q+1}}}{K^{\frac{1-qc}{q+1}}(1-cq)} & \text{if } cq < 1 \\
		T^\frac{q}{q+1} \left(\log  \frac{T}{K}\right)^{\frac{1}{q+1}} & \text{if } cq = 1\\
		\frac{T^\frac{q}{q+1}}{cq-1} & \text{if } cq > 1
	\end{cases} \\
	& \le \BigO \left(K \begin{cases}
		\frac{T^{\frac{1-cq+q}{q+1}}}{K^{\frac{1-qc}{q+1}}(1-cq)} & \text{if } cq < 1 \\
		T^\frac{q}{q+1} \left(\log  \frac{T}{K}\right)^{\frac{1}{q+1}} & \text{if } cq = 1\\
		\frac{T^\frac{q}{q+1}}{\min\{1,cq-1\}} & \text{if } cq > 1
	\end{cases} \right), \qquad \forall q \in [0,1].
\end{align*}
For the case $c \in (0,1)$, we consider the first case only and minimize over $q$:
\begin{align*}
	R_{\bm{\mu}} \le \BigO \left( K \min_{q \in [0,1]} \frac{T^{\frac{1-cq+q}{q+1}}}{K^{\frac{1-qc}{q+1}}(1-cq)} \right) \le \BigO \left( \frac{K^{\frac{1+c}{2}}T^{1-\frac{c}{2}}}{1-c} \right),
\end{align*}
for sufficiently large $T \gg K$.
For the case $c=1$, it is simple to prove that the case $cq=1$ leads to the smallest regret:
\begin{align*}
	R_{\bm{\mu}} \le K T^\frac{1}{c+1} \left(\log  \frac{T}{K}\right)^{\frac{c}{c+1}}.
\end{align*}
Finally, for the case $c \in (1,+\infty)$, we have to consider all the three cases:
\begin{align*}
	R_{\bm{\mu}} & \le \BigO \left( K \begin{cases}
		\min_{q \in [0, 1/c)}  \frac{T^{\frac{1-cq+q}{q+1}}}{K^{\frac{1-qc}{q+1}}(1-cq)}  \\
		T^\frac{1}{c+1} \left(\log  \frac{T}{K}\right)^{\frac{c}{c+1}}\\
		\min_{q \in (1/c,1]} \frac{T^\frac{q}{q+1}}{\min\{1,cq-1\}} 
	\end{cases}\right) = K T^\frac{1}{c+1} \left(\log  \frac{T}{K}\right)^{\frac{c}{c+1}}.
\end{align*}

\newpage
\section{Technical Lemmas}

\begin{lemma}\label{lemma:sumWithFloor}
	Let $M \ge 3$, and let $f: \Nat \rightarrow \Reals$, and $\beta \in (0, 1)$. Then it holds that:
	\begin{align*}
		\sum_{j=3}^{M} f(\lfloor \beta j \rfloor) \le \left\lceil \frac{1}{\beta} \right\rceil \sum_{l=\lfloor 3 \beta \rfloor}^{\lfloor \beta M \rfloor} f(l).
	\end{align*}
\end{lemma}

\begin{proof}
	We simply observe that the minimum value of $\lfloor \beta j \rfloor$ is $\lfloor 3\beta  \rfloor$ and its maximum value is $\lfloor \beta M \rfloor$. Each element $\lfloor \beta j \rfloor$ changes value at least one time every $\left\lceil \frac{1}{\beta} \right\rceil$ times.
\end{proof}

\begin{lemma}\label{lemma:boundUpsilonMax}
	Under Assumption~\ref{ass:decrDeriv}, it holds that:
	\begin{align*}
		\max_{\substack{(N_{i,T})_{i \in [K]} \\ N_{i,T} \ge 0, \sum_{i \in [K]}N_{i,T} = T}}   \;\;\sum_{i \in [K]} \sum_{l=1}^{N_{i,T}-1} \gamma_i(l)^q \le K \Upsilon_{\bm{\mu}}\left( \left\lceil \frac{T}{K} \right\rceil, q \right).
	\end{align*}
\end{lemma}

\begin{proof}
First of all, by definition of cumulative increment, we have that for every $(N_{i,T})_{i \in [K]}$:
\begin{align*}
\sum_{i \in [K]} \sum_{l=1}^{N_{i,T}-1} \gamma_i(l)^q \le \sum_{i\in [K]} \Upsilon_{\bm{\mu}}(N_{i,T},q).
\end{align*}
We now claim that there exists an optimal assignment of $N_{i,T}^*$ are such that $|N_{i,T}^*-N_{i',T}^*| \le 1$ for all $i,i' \in [K]$. By contradiction, suppose that the only optimal assignments are such that there exists a pair $i_1,i_2 \in [K]$ such that $\Delta \coloneqq N_{i_2,T}^* - N_{i_1,T}^* > 1$. In such a case, we have:
	\begin{align*}
		 \Upsilon_{\bm{\mu}}\left( N_{i_1,T}^*, q \right) + \Upsilon_{\bm{\mu}}\left( N_{i_2,T}^*, q \right) & = 2\Upsilon_{\bm{\mu}}\left( N_{i_1,T}^*, q \right) + \sum_{j = 1}^{\Delta} \max_{i \in [K]}\gamma_{i}(N_{i_1,T}^*+l-1)^q  \\
		 & \le 2\Upsilon_{\bm{\mu}}\left( N_{i_1,T}^*, q \right) +  \sum_{j = 0}^{\lceil \Delta / 2 \rceil}  \max_{i \in [K]} \gamma_{i}(N_{i_1,T}^*+l-1)^q  + \sum_{j = 1}^{\lfloor \Delta / 2 \rfloor}  \max_{i \in [K]} \gamma_{i}(N_{i_1,T}^*+l-1)^q \\
		 & = \Upsilon_{\bm{\mu}}\left( N_{i_1,T}^* + {\lceil \Delta / 2 \rceil}, q \right) + \Upsilon_{\bm{\mu}}\left( N_{i_1,T}^* + \lfloor \Delta / 2 \rfloor, q \right),
	\end{align*}
	where the inequality follows from Assumption~\ref{ass:decrDeriv}. By redefining $\widetilde{N}_{i_1,T}^* \coloneqq N_{i_1,T}^* + {\lfloor \Delta / 2 \rfloor}$ and $\widetilde{N}_{i_2,T}^* \coloneqq N_{i_1,T}^* + {\lceil \Delta / 2 \rceil}$, we have that $\widetilde{N}_{i_1,T}^* + \widetilde{N}_{i_2,T}^* = {N}_{i_1,T}^* + {N}_{i_2,T}^*$ and $|\widetilde{N}_{i_1,T}^* -\widetilde{N}_{i_2,T}^*| \le 1$. Thus, we have found a better solution to the optimization problem, contradicting the hypothesis. Since the optimal assignment fulfills $|N_{i,T}^*-N_{i',T}^*| \le 1$, it must be that $N_{i,T}^* \le \left\lceil \frac{T}{K} \right\rceil$ for all $i \in [K]$.
\end{proof}

\begin{restatable}[]{lemma}{boundDeriv}\label{lemma:boundDeriv}
	Under Assumptions~\ref{ass:incr} and~\ref{ass:decrDeriv}, for every $i\in [K]$, $k,k' \in \Nat$ with $k'<k$, for both rested and restless bandits, it holds that:
	\begin{align*}
		\gamma_i(k) \le \frac{\mu_i(k) - \mu_i(k')}{k-k'}.
	\end{align*}
\end{restatable}

\begin{proof}
	Using Assumption~\ref{ass:decrDeriv}, we have:
	\begin{align*}
		\gamma_i(k) = \frac{1}{k-k'} \sum_{l=k'}^{k-1} \gamma_i(k) \le \frac{1}{k-k'} \sum_{l=k'}^{k-1} \gamma_i(l) = \frac{1}{k-k'} \sum_{l=k'}^{k-1} \left( \mu_i(l+1) - \mu_i(l) \right)  = \frac{\mu_i(k) - \mu_i(k')}{k-k'},
	\end{align*}
	where the first inequality comes from the concavity of the reward function, and the second equality from the definition of increment.
\end{proof}

\begin{lemma}\label{lemma:tecSum} 
Let $a,b \in \Nat$ and let $f: [a,b] \rightarrow \Reals$. If $f$ is monotonically non-decreasing function, then:
\begin{align*}
	\sum_{n=a}^b f(n) \le \int_{x = a}^b f(x) \de x + f(b) \le \int_{x = a}^{b+1} f(x) .
\end{align*}
If $f$ is monotonically non-increasing, then:
\begin{align*}
	\sum_{n=a}^b f(n) \le f(a) + \int_{x = a}^{b} f(x) \de x \le \int_{x = a-1}^{b} f(x) \de x.
\end{align*}
\end{lemma}

\begin{proof}
	Let us consider the intervals $I_i = [x_{i-1},x_{i}]$ with $x_{0} = a$ and $x_{i} = x_{i-1} + 1$ for $i \in [b-a]$. If $f$ is monotonically non-decreasing, we have that for all $i \in [b-a]$ and $x \in I_i$ it holds that $f(x) \ge f(x_{i-1})$ and consequently $\int_{I_i} f(x) \de x \ge f(x_{i-1}) \mathrm{vol}(I_{i}) = f(x_{i-1})$. Thus:
	\begin{align*}
		\sum_{n=a}^b f(n) = \sum_{i=1}^{b-a} f(x_{i-1}) + f(b) \le \sum_{i=1}^{b-a} \int_{I_i} f(x) \de x + f(b) = \int_{x=a}^b f(x) \de x + f(b).
	\end{align*}
	Recalling that $f(b) \le \int_{x=b}^{b+1} f(x) \de x$, we get the second inequality.
	Conversely, if $f$ is monotonically non-increasing, then for all $i \in [b-a]$ and $x \in I_i$, it holds that $f(x) \ge f(x_{i})$ and consequently $\int_{I_{i}} f(x) \de x \ge f(x_{i})$. Thus:
	\begin{align*}
		\sum_{n=a}^b f(n) = f(a) + \sum_{i=1}^{b-a} f(x_{i})  \le f(a) + \sum_{i=1}^{b-a} \int_{I_i} f(x) \de x = f(a) + \int_{x=a}^b f(x) \de x .
	\end{align*}
	Recalling that $f(a) \le \int_{x=a-1}^{a} f(x) \de x$, we get the second inequality.
\end{proof}

\begin{thr}[H\"oeffding-Azuma’s inequality for weighted martingales] \label{lemma:hoeffding}
Let $\mathcal{F}_1 \subset \dots \subset \mathcal{F}_n$ be a filtration and $X_1, \dots, X_n$ be real random variables such that $X_t$ is $\mathcal{F}_t$-measurable, $\E[X_t|\mathcal{F}_{t-1}]=0$ (\ie a martingale difference sequence), and $\E[\exp(\lambda X_t)|\mathcal{F}_{t-1}]\le \exp\left( \frac{\lambda^2 \sigma^2}{2} \right)$ for any $\lambda >0$ (\ie $\sigma^2$-subgaussian). Let $\alpha_1, \dots, \alpha_n$ be non-negative real numbers. Then,  for every $\kappa \ge 0$ it holds that:
\begin{align*}
\Pr \left( \left| \sum_{t=1}^n \alpha_t X_t \right| >\kappa\right) \le 2 \exp \left( -\frac{\kappa^2}{2 \sigma^2 \sum_{t=1}^n \alpha_i^2}  \right).
\end{align*}
%
%
%
%
\end{thr}
\begin{proof}
	It is a straightforward extension of Azuma-H\"oeffding inequality for subgaussian random variables. We apply the Chernoff's method for some $s > 0$:
	\begin{align*}
	\Pr \left(  \sum_{t=1}^n \alpha_t X_t >\kappa\right) & = \Pr \left(  e^{ s \sum_{t=1}^n \alpha_t X_t}> e^{s\kappa}\right) \le \frac{\E\left[  e^{ s \sum_{t=1}^n \alpha_t X_t} \right]}{ e^{s\kappa}},
	\end{align*}
	where the last inequality follows from the application of Markov's inequality. We use the martingale property to deal with the expectation. By the law of total expectation, we have:
	\begin{align*}
	\E\left[  e^{ s \sum_{t=1}^n \alpha_t X_t} \right] = \E\left[  e^{ s \sum_{t=1}^{n-1} \alpha_t X_t} \E \left[e^{ s \alpha_n X_n}  \rvert \mathcal{F}_{t-1} \right] \right].
	\end{align*}
	Using now the subgaussian property, we have:
	\begin{align*}
	\E \left[e^{ s \alpha_n X_n}  \rvert \mathcal{F}_{t-1} \right]  \le \exp \left( \frac{s^2\alpha_n^2 \sigma^2}{2} \right).
	\end{align*}
	An inductive argument, leads to:
	\begin{align*}
	\E\left[  e^{ s \sum_{t=1}^n \alpha_t X_t} \right] \le  \exp \left( \frac{s^2 \sigma^2}{2} \sum_{t=1}^n \alpha_n^2 \right).
	\end{align*}
	Thus, minimizing \wrt $s > 0$, we have:
	\begin{align*}
		\Pr \left(  \sum_{t=1}^n \alpha_t X_t >\kappa\right)  \le\min_{s \ge 0} \, \exp \left( \frac{s^2 \sigma^2}{2} \sum_{t=1}^n \alpha_n^2 - s\kappa \right) =   \exp\left( -\frac{\kappa^2}{2  \sigma^2 \sum_{t=1}^n \alpha_n^2}\right),
	\end{align*}
	being the minimum attained by $s = \frac{\kappa}{\sigma^2 \sum_{t=1}^n \alpha_n^2}$. The reverse inequality can be derived analogously. A union bound completes the proof.
\end{proof}

\begin{lemma}\label{lemma:lemmaBoundsGamma}
	Let $\Upsilon_{\bm{\mu}}(T,q)$ be as defined in Equation~\eqref{eq:magicQuantity} for some $q \in [0,1]$. Then, for all $i \in [K]$ and $l \in \Nat$ the following statements hold:
	\begin{itemize}
		\item if $\gamma_i(l) \le b e^{-cl}$, then $\Upsilon_{\bm{\mu}}(T,q) \le  \BigO\left( b^{q} \frac{e^{-cq}}{cq} \right)$;
		\item if $\gamma_i(l) \le b l^{-c}$  with $cq > 1$, then $\Upsilon_{\bm{\mu}}(T,q) \le \BigO\left( \frac{b^q}{cq - 1} \right)$;
		\item if $\gamma_i(l) \le b l^{-c}$  with $cq = 1$, then $\Upsilon_{\bm{\mu}}(T,q) \le \BigO\left(  b^q \log T\right)$;
		\item if $\gamma_i(l) \le b l^{-c}$  with $cq < 1$, then $\Upsilon_{\bm{\mu}}(T,q) \le \BigO\left( b^q\frac{T^{1-cq}}{1-cq} \right)$.
	\end{itemize}
\end{lemma}

\begin{proof}
The proofs of all the statements are obtained by bounding the summation defining $\Upsilon_{\bm{\mu}}(T,q)$ with the corresponding integrals, as in Lemma~\ref{lemma:tecSum}. Let us start with $\gamma_i(l) \le b e^{-cl}$:
\begin{align*}
	\Upsilon_{\bm{\mu}}(T,q) = \sum_{l=1}^T \gamma_i(l)^q \le  b^qe^{-cq} + \int_{x=1}^{T} b^qe^{-cqx} \de x \le b^qe^{-cq} + \frac{b^q}{cq} e^{-cq} = \BigO\left( b^{q} \frac{e^{-cq}}{cq} \right).
\end{align*}
We now move to $\gamma_i(l) \le b l^{-c}$. If $cq < 1$, we have:
\begin{align*}
\Upsilon_{\bm{\mu}}(T,q) = \sum_{l=1}^T \gamma_i(l)^q \le b^q + \int_{x=1}^{T} b^q x^{-cq} \de x = b^q + \frac{b^q}{cq - 1} = \BigO\left( \frac{b^q}{cq - 1} \right).
\end{align*}
For $cq=1$, we obtain:
\begin{align*}
\Upsilon_{\bm{\mu}}(T,q) = \sum_{l=1}^T \gamma_i(l)^q \le b^q + \int_{x=1}^{T} \frac{b^q}{x} \de x = b^q + b^q \log T = \BigO\left(  b^q \log T\right).
\end{align*}
Finally, for $cq < 1$, we have:
\begin{align*}
	\Upsilon_{\bm{\mu}}(T,q) = \sum_{l=1}^T \gamma_i(l)^q \le b^q + \int_{x=1}^{T} b^q x^{-cq} \de x = b^q + b^q\frac{T^{1-cq}}{1-cq} = \BigO\left( b^q\frac{T^{1-cq}}{1-cq} \right).
\end{align*}
The results of Table~\ref{tab:rates} are obtained by setting $b=1$.
\end{proof}

\section{Efficient Update}\label{apx:efficient}
Under the assumption that the window size depends on the number of pulls only and that $0 \le h({n+1}) - h({n}) \le 1$, we can employ the following efficient $\mathcal{O}(1)$ update for \algrested and \algrestless. Denoting with $n$ the number of pulls of arm $i$, we update the estimator at every time step $t \in [T]$ as:
\begin{align*}
	\widehat{\mu}_{i}^{h(n)}(t) = \frac{1}{h(n)} \left( a_n + \frac{t(a_n - b_n)}{h(n)}  - \frac{c_n - d_n}{h(n)} \right),
\end{align*}
where the following sequences are updated only when the arm is pulled:
\begin{align*}
	& a_{n} = \begin{cases}
					a_{n-1} + r_i(n)-r_i(n-h(n)) & \text{if } h({n}) = h({n-1}) \\
					a_{n-1} + r_i(n) & \text{otherwise}
				\end{cases}, \\
	&  b_{n} = \begin{cases}
					b_{n-1} +r_i(n-h(n)) - r_i(n-2h(n)) & \text{if } h({n}) = h({n-1}) \\
					b_{n-1} + r_i(n-2 h(n)+1) & \text{otherwise}
				\end{cases}, \\
	&  c_{n} = \begin{cases}
					c_{n-1} +nr_i(n) - (n-h(n))r_i(n-h(n)) & \text{if } h({n}) = h({n-1}) \\
					c_{n-1} + nr_i(n) & \text{otherwise}
				\end{cases}, \\
	&  d_{n} = \begin{cases}
					d_{n-1} +n r_i(n-h(n)) -(n-h(n)) r_i(n-2h(n))& \text{if } h({n}) = h({n-1}) \\
					d_{n-1} + (n-h(n))r_i(n-2h(n)+1) + b_n & \text{otherwise}
				\end{cases},
\end{align*}
where we have abbreviated $r_i(n) \coloneqq R_{t_{i,n}}$.

\section{Experimental Setting and Additional Results}\label{apx:experiments}

\subsection{Parameter Setting}
The choices of the parameters of the algorithms we compared \texttt{R-less/ed-UCB} with are the following:
\begin{itemize}
	\item \texttt{Rexp3}: $V_T = K$ since in our experiments we consider the reward of each arm to evolve from $0$ to $1$, thus the maximum global variation possible is equal the number of arms of the bandit; $\gamma = \min \left\{ 1, \sqrt{\frac{K\log{K}}{(e-1)\Delta_T}} \right\}$, $\Delta_T = \lceil (K\log{K})^{1/3} (T/V_T)^{2/3} \rceil$ as recommended by~\citet{besbes2014stochastic};
	\item \texttt{KL-UCB}: $c = 3$ as required by the theoretical results on the regret provided by~\citet{garivier2011kl};
	\item \texttt{Ser4}: according to what suggested by~\citet{allesiardo2017non} we selected $\delta=1/T$, $\epsilon=\frac{1}{KT}$, and $\phi = \sqrt{\frac{N}{TK\log({KT})}}$;
	\item \texttt{SW-UCB}: as suggested by~\citet{garivier2011upper} we selected the sliding-window $\tau = 4\sqrt{T\log{T}}$ and the constant $\xi = 0.6$;
	\item \texttt{SW-KL-UCB} as suggested by~\citet{garivier2011upper} we selected the sliding-window $\tau = \sigma^{-4/5}$;
	\item \texttt{SW-TS}: as suggested by~\citet{trovo2020sliding} for the smoothly changing environment we set $\beta = 1/2$ and sliding-window $\tau = T^{1-\beta} = \sqrt{T}$.
\end{itemize}

\subsection{IMDB Experiment}\label{apx:imdb}
We created a bandit environment in which each of the classification algorithms is an arm of the bandit.
The interaction for each round $t \in T$ of the real-world experiment is composed by the following:
\begin{itemize}
	\item the agent decides to pull arm $I_t$;
	\item a random point $x_t$ of the IMDB dataset is selected and supplied to the classification algorithm associated to arm $I_t$;
	\item the \quotes{base} algorithm classifies the sample, \ie~it provides the prediction $\hat{y}_t \in \{0,1\}$ for the selected sample $x_t$;
	\item the environment generates the reward comparing the prediction $\hat{y}_t$ to the target class $y_t$ using the following function $R_t = 1 - |y_t - \hat{y}_t|$;
	\item the base algorithm is updated using $(x_t, y_t)$;
\end{itemize}
Since the base algorithms are trained only if their arm is selected, this is a problem which belongs to the rested scenario.

For the classification task we decided to employ:
\begin{itemize}
	\item $2$ Online Logistic Regression (LR) methods with different schemes used for the learning rate $\lambda_t$;
	\item $5$ Neural Networks (NNs) different in terms of shape and number of neurons
\end{itemize}
Specifically, we adopt a decreasing scheme for the learning rate of $\lambda_t = \frac{\beta}{t}$ (denoted with \texttt{LR}$(t)$ from now on) and a constant learning rate $\lambda_t = \beta$ (denoted as \texttt{LR} from now on).
Moreover, the NNs use as activation functions the rectified linear unit, \ie $relu(x) = \max(0,x)$, a constant learning rate $\alpha = 0.001$ and the \quotes{adam} stochastic gradient optimizer for fitting.
Two of the chosen nets have only one hidden layer, with $1$ and $2$ neurons, respectively, the third net has $2$ hidden layer, with $2$ neurons each, and two nets have $3$ layers with $2,2,2$ and $1,1,2$ neurons, respectively.
We refer to a specific NN denoting in curve brackets the cardinalities of the layers, \eg the one having $2$ layer with $2$ neurons each is denoted by \texttt{NN}$(2,2)$.

We analyzed their global performance on the IMDB dataset by averaging $1,000$ independent runs in which each strategy is sequentially fed with all the available $50,000$ samples.
The goal was to determine, at each step, the value of the payoff $\mu_i(n)$. Figure~\ref{fig:imdb_learning_curves} provides the average learning curves of the selected algorithms.
As we expected, from a qualitative perspective, the average learning curves are increasing and concave, however, due to the limited number of simulations, Assumptions~\ref{ass:incr} and~\ref{ass:decrDeriv} are not globally satisfied.

We also perform an experiment using only \texttt{LR}$(t)$ and \texttt{LR} as arms.
Figure~\ref{fig:imdb_old} reports the result of a run of the MAB algorithms over the IMDB scenario.
The analogy between this result and the one of the $2$-arms synthetic rested bandit (Figure~\ref{fig:rested_2arms_regrets}) is clear, indeed \algrested outperforms the other baselines when the learning curves of the base algorithms at some point intersects one another. 

\begin{figure*}
	\vspace{.25cm}
	\begin{minipage}{.48\textwidth}
		\centering
		\subfloat{\scalebox{1}{\includegraphics{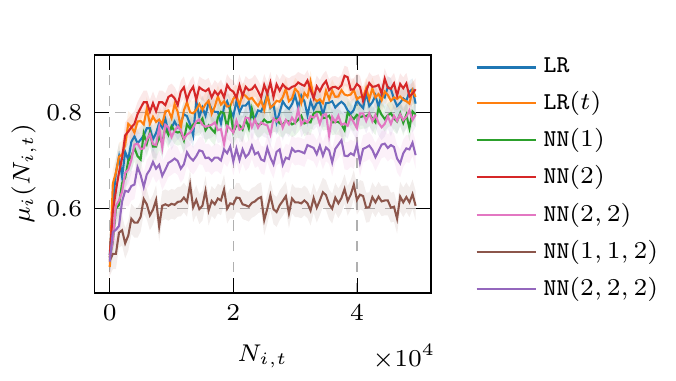} \label{fig:imdb_learning_curves}}}
		\captionof{figure}{Empirical learning curves of the classification algorithms (arms) of the IMDB experiment}
	\end{minipage}%
	\hfill
	\begin{minipage}{.48\textwidth}
		\centering
		\subfloat{\scalebox{1}{\includegraphics{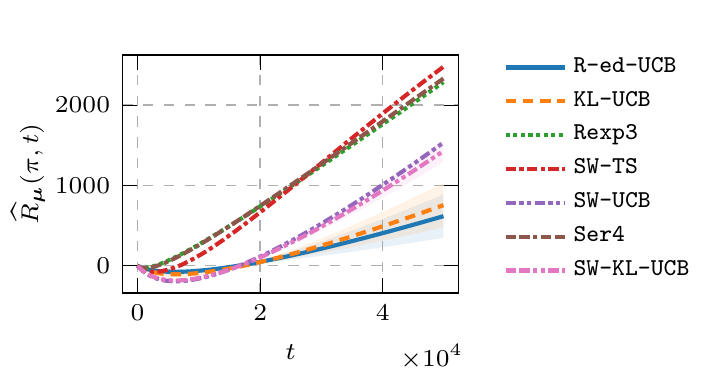} \label{fig:imdb_old}}}
		\captionof{figure}{Cumulative regret in a $2$-arms online model selection on IMDB dataset ($30$ runs, $95$\% c.i.).}
	\end{minipage}
	\end{figure*}

\subsection{Pulls of each arm}

Figure~\ref{fig:pulls} presents the average number of pulls for each arm for each one of the algorithm analysed in the synthetic experiments of Section~\ref{sec:experiments}.
Figure~\ref{fig:restless_15arms_pulls} shows how \algrestless{} is able to identify and discard the majority of the suboptimal arms using a few pulls, and it is second only to \texttt{SW-TS}, which seems to commit to a single arm which turns out to be the optimal one (arm $13$).
Figure~\ref{fig:rested_15arms_pulls} shows that \algrested{} explored arms $13$ and $1$ more than the others, which are respectively the best and the second-best, and most likely needs a longer time horizon to select which one is the best among the twos.
Figure~\ref{fig:rested_2arms_pulls} highlights the fact that \algrested{} undoubtedly identified which arm is the best (arm $2$), while \texttt{KL-UCB}, \texttt{SW-UCB}, \texttt{SW-KL-UCB} do not identify the best arm. \texttt{Ser4}, \texttt{Rexp3} and \texttt{SW-TS} pulled the best arm slightly more than $50\%$ of the times, paying the already discussed initial learning phase.

\begin{figure}
	\centering
	\subfloat{\scalebox{0.75}{\includegraphics{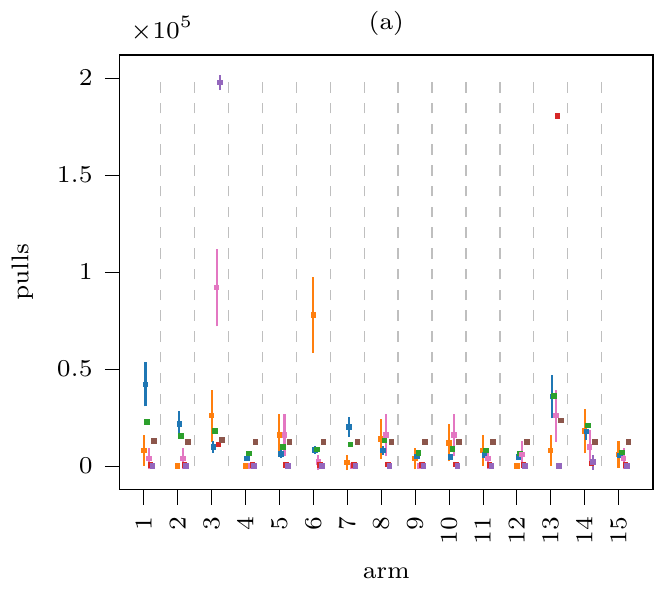} \label{fig:restless_15arms_pulls}}}
	\hfill
	\subfloat{\scalebox{0.75}{\includegraphics{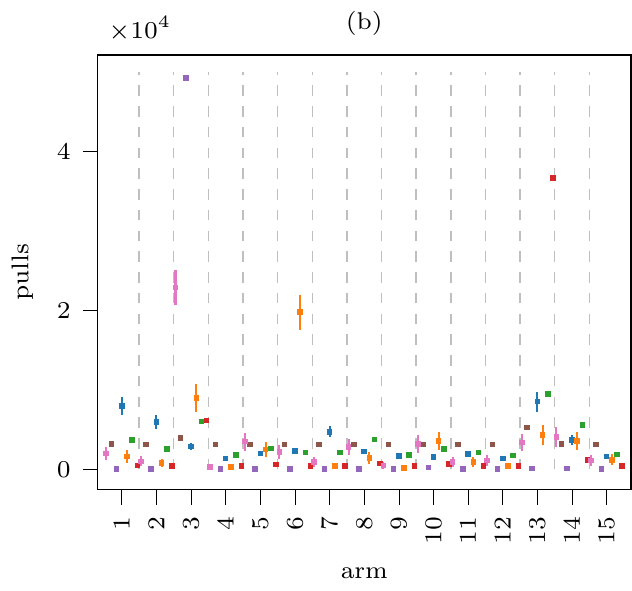} \label{fig:rested_15arms_pulls}}}
	\hfill
	\subfloat{\scalebox{0.75}{\includegraphics{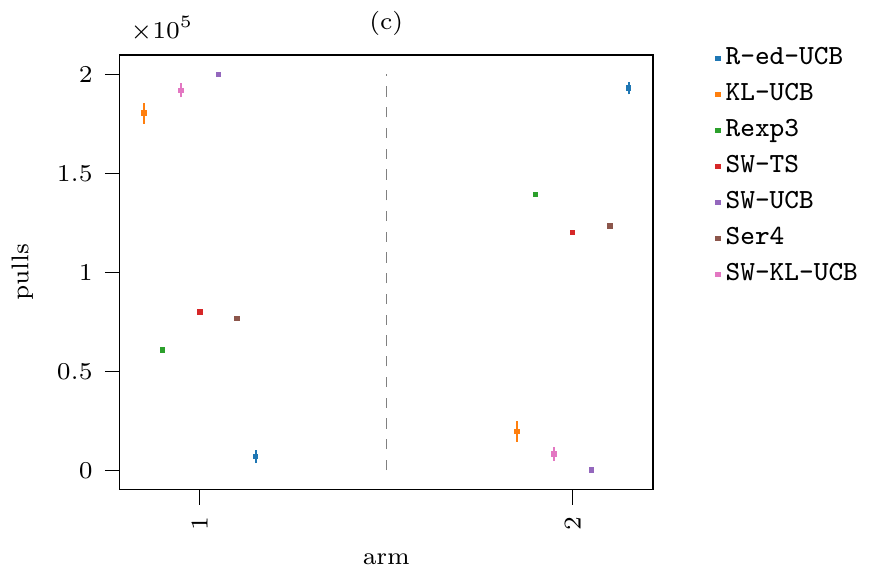} \label{fig:rested_2arms_pulls}}}
	\caption{Average number of pulls: (a)~$15$ arms \rless, (b)~$15$ arms \red, (c)~$2$ arms \red.} \label{fig:pulls}
\end{figure}

\subsection{Additional Experimental Results}

We evaluated the performance of the algorithms over $50$ different bandits with $K \in \{2, \dots, 15\}$ randomly generated arms over a time horizon of $T = 200,000$ rounds.
We averaged the run of each algorithm on a single scenario over $10$ independent experiments and compared the expected value of the ranking of the considered algorithms in order to draw up a leaderboard: in every scenario we ranked the algorithms based on their empirical regret, giving the first placement to the one with the lowest value.
We report the summarized results of the rank of the algorithms averaged over the $50$ experiments in Table~\ref{tab:ranking}.
In the rested case \algrested is among the worse ones ($4.98$ on average), however this is again due to the fact that on average the algorithm is not superior to the baselines, which conversely do not provide any theoretical guarantees in some specific settings (see the $2$ arm experiment).
In the restless case \algrestless achieves a worse-than-average performance, probably influenced by the characteristics of the randomly generated bandits.
Due to this unsatisfactory results, we propose a slight modification of the \algrestless upper bound as follows:
$$\widehat{\mu}^{\text{\rless},h}_i(t) \coloneq  \frac{1}{h} \sum\limits_{l=N_{i,t-1}-h+1}^{N_{i,t-1}} \bigg( R_{t_{i,l}} + \textcolor{vibrantRed}{( t - t_{i,l})} \frac{R_{t_{i,l}} - R_{t_{i,l-h}}}{\textcolor{vibrantRed}{t_{i,l}- t_{i,l-h}}} \bigg),$$
which we call \algrestless-\texttt{H} to denote it is an heuristic method, \ie not having theoretical results on the regret.
While the performance of the heuristic seems good in practice (it achieves the best overall result), its downside is that the theoretical guarantees on the regret will have to be reconsidered.

\begin{table}[t!]
	\centering
	\caption{Ranking of the algorithms ($50$ bandits, $10$ runs, $95\%$ c.i.~in brackets).}
	\begin{tabular}{l|c|c|c}
		\
		\bfseries{Algorithm} & \bfseries{Rested Setting Ranking} & \bfseries{Restless Setting Ranking} & \bfseries{Restless Setting Ranking Heuristic}\\
		\hline
		\texttt{R-ed-UCB} 		& $4.98 \ (0.34)$ & $-$ & $-$ \\
		\texttt{R-less-UCB} 	& $-$ & $5.14 \ (0.38)$ & $-$ \\
		\texttt{R-less-UCB-H} 	& $-$ & $-$ & $1.90 \ (0.30)$ \\
		\texttt{KL-UCB} 		& $2.56 \ (0.43)$ & $2.54 \ (0.34)$ & $2.46 \ (0.31)$ \\
		\texttt{Rexp3} 			& $5.10 \ (0.26)$ & $5.20 \ (0.26)$ & $6.08 \ (0.16)$ \\
		\texttt{SW-TS} 			& $2.84 \ (0.35)$ & $2.86 \ (0.39)$ & $4.76 \ (0.19)$ \\
		\texttt{SW-UCB} 		& $2.12 \ (0.44)$ & $2.58 \ (0.47)$ & $3.08 \ (0.30)$ \\
		\texttt{Ser4} 			& $6.84 \ (0.15)$ & $6.60 \ (0.28)$ & $6.66 \ (0.18)$ \\
		\texttt{SW-KL-UCB} 		& $3.56 \ (0.38)$ & $3.08 \ (0.45)$ & $3.06 \ (0.48)$ \\
	\end{tabular}
	\normalsize
	\label{tab:ranking}
\end{table}

\end{document}